\documentclass[british, twoside]{article}
  
\usepackage[accepted]{aistats2021}
 \usepackage{amsmath}
 \usepackage{graphicx}
\usepackage{amssymb}
\usepackage{epstopdf}
\usepackage{amsthm}
\usepackage[round]{natbib}

\usepackage[utf8]{inputenc}
\usepackage{parskip}    %

\usepackage{accents}

\usepackage{amsfonts}       %

\usepackage[utf8]{inputenc} %
\usepackage[T1]{fontenc}    %
\RequirePackage{doi}
\usepackage{url}            %
\usepackage{booktabs}       %
\usepackage{nicefrac}       %
\usepackage{microtype}
\usepackage[justification=centering]{subfig}
\usepackage{enumitem}
\usepackage{wrapfig}
\usepackage{mathtools}
\usepackage{pgffor}
\usepackage{grffile}
\usepackage{caption}
\usepackage{tikz}
\usepackage{subfiles}
\usepackage{varioref}
\usepackage{hyperref}
\usepackage{etoolbox}
\usepackage{fixltx2e}
\usepackage{todonotes}
\usepackage[export]{adjustbox}

\usepackage{multirow}

\usetikzlibrary{patterns}
\usetikzlibrary{bayesnet}
\usetikzlibrary{arrows,decorations.markings}
\tikzstyle{disent_latent} = [circle,pattern=north east lines, pattern color=black!20,draw=black,inner sep=1pt,
minimum size=20pt, font=\fontsize{10}{10}\selectfont, node distance=1]

\tikzstyle{obs_det} = [latent, diamond, fill=gray!25]

\DeclareMathOperator{\expect}{\mathbb{E}}

\DeclareMathOperator{\ELBO}{\mathcal{L}}
\DeclareMathOperator{\KL}{\mathrm{KL}}
\DeclareMathOperator{\GG}{\mathrm{GG}}

\newcommand{\norm}[1]{\left | #1\right |}

\newcommand{\vecto}[1]{\boldsymbol{\mathbf{#1}}}
\renewcommand{\v}{\vecto}
\newtheorem{theorem}{Theorem}
\newtheorem{proposition}{Proposition}

\makeatletter
\newif\ifnobrackets
\renewcommand\@cite[2]{\ifnobrackets\else[\fi{#1\if@tempswa , #2\fi}\ifnobrackets\else]\fi\nobracketsfalse}
\newcommand\nbcite{\nobracketstrue\cite}
\makeatother

\usepackage{hyperref}
        
\DeclarePairedDelimiterX\MeijerM[3]{\lparen}{\rparen}%
{\begin{smallmatrix}#1 \\ #2\end{smallmatrix}\delimsize\vert\,#3}

\fancypagestyle{appendixpage}{
    \fancyhead{}
    \renewcommand{\headrulewidth}{0pt}
    \renewcommand{\footrulewidth}{0pt}}

\begin{document}
\runningauthor{Alexander Camuto\textbf{*}, Matthew Willetts\textbf{*}, Brooks Paige, Chris Holmes Stephen Roberts}
\setlength\abovedisplayskip{6pt plus 1pt minus 6pt}
\setlength\belowdisplayskip{6pt plus 1pt minus 6pt}

\twocolumn[

\aistatstitle{Learning Bijective Feature Maps for Linear ICA}

\aistatsauthor{Alexander Camuto\textbf{*}$^{,1,3}$ \hspace{1.8cm} Matthew Willetts\textbf{*}$^{,1,3}$}
\aistatsauthor{Brooks Paige$^{2,3}$ \hspace{1.2cm}
Chris Holmes$^{1,3}$ \hspace{1.2cm} Stephen Roberts$^{1,3}$}
\aistatsaddress{${}^{1}$University of Oxford \hspace{0.4cm} ${}^{2}$University College London \hspace{0.4cm} ${}^{3}$Alan Turing Institute}]

\begin{abstract}
Separating high-dimensional data like images into independent latent factors, i.e independent component analysis (ICA), remains an open research problem.
As we show, existing probabilistic deep generative models (DGMs), which are tailor-made for image data, underperform on non-linear ICA tasks.
To address this, we propose a DGM which combines bijective feature maps with a linear ICA model to learn interpretable latent structures for high-dimensional data.
Given the complexities of jointly training such a hybrid model, we introduce novel theory that constrains linear ICA to lie close to the manifold of orthogonal rectangular matrices, the Stiefel manifold.
By doing so we create models that converge quickly, are easy to train, and achieve better unsupervised latent factor discovery than flow-based models, linear ICA, and Variational Autoencoders on images.

\end{abstract}

\section{Introduction}
\label{intro}

In linear Independent Component Analysis (ICA), data is modelled as having been created from a linear mixing of independent latent \textit{sources}~\citep{Cardoso1989, Cardoso1989a, Cardoso1997, Comon1994}.
The canonical problem is blind source separation;
the aim is to estimate the original sources of a mixed set of signals by learning an \textit{unmixing} matrix, which when multiplied with data recovers the values of these sources.
While linear ICA is a powerful approach to unmix signals like sound \citep{Everson2001}, it has not been as effectively developed for learning compact representations of high-dimensional data like images, where assuming linearity is limiting.
Non-linear ICA methods, which assume non-linear mixing of latents, offer better performance on such data.

In particular, flow-based models have been proposed as a non-linear 
approach to \textit{square} ICA, where we assume the dimensionality of our latent source space is the same as that of our data \citep{Deco1995, Dinh2015}.
Flows parameterise a bijective mapping between data and a feature space of the same dimension and can be trained via maximum likelihood for a chosen base distribution in that space. 
While these are powerful generative models, for image data one typically wants fewer latent variables than the number of pixels in an image.
In such situations, we wish to learn a non-square (dimensionality-reducing) ICA representation.

\begin{figure*}[t!]
   \centering
    
    \subfloat[Sample Images]{

\begin{tabular}[Samples]{c}
\includegraphics[width=0.06\textwidth]{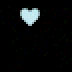}
\includegraphics[width=0.06\textwidth]{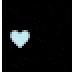}
\\
\end{tabular}}
\hspace{0.0cm}
\subfloat[VAE Sources]{
\begin{tabular}[Samples]{c}
\includegraphics[width=0.3\textwidth]{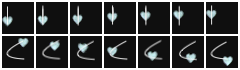}
\end{tabular}
\label{fig:vae_curves}
}
\hspace{0.0cm}
\subfloat[Bijecta Sources]{
\begin{tabular}[Samples]{c}
\includegraphics[width=0.3\textwidth]{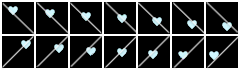}
\end{tabular}}

\hspace{0mm}
    \caption{
    Here we take a dSprites heart and, using a randomly sampled affine transformation, move it around a black background (a).
    The underlying sources of the dataset are \textit{affine} transformations of the heart. 
    In (b-c) images in the center correspond to the origin of the learnt source space.
    Images on either side correspond to linearly increasing values along one of the learnt latent sources whilst the other source remains fixed.
    Bijecta (c) has learned affine transformations as sources (white diagonals), whereas a VAE (with ICA-appropriate prior) (b) has learned non linear transforms (white curves). The VAE has not discovered the underlying latent sources. 
    }
    \label{fig:dpsrites-demo}
\end{figure*}

In this work, we highlight the fact that existing probabilistic deep generative models (DGMs), in particular Variational Autoencoders (VAEs), underperform on non-linear ICA tasks.
As such there is a real need for a probabilistic DGM that can perform these tasks. 
To address this we propose a novel methodology for performing non-square non-linear ICA using a model, termed \textit{Bijecta}, with two jointly trained parts: a highly-constrained non-square linear ICA model, operating on a feature space output by a bijective flow. 
The bijective flow is tasked with learning a representation for which linear ICA is a good model.
It is as if we are \textit{learning the data} for our ICA model.

We find that such a model fails to converge when trained naively with no constraints. 
To ensure convergence, we introduce novel theory for the parameterisation of decorrelating, non-square ICA matrices that lie close to the Stiefel manifold \citep{Stiefel1935}, the space of orthonormal rectangular matrices.
We use this result to introduce a novel non-square linear ICA model that uses Johnson-Lidenstrauss projections (a family of randomly generated matrices).
Using these projections, Bijecta successfully induces dimensionality reduction in flow-based models and scales non-square non-linear ICA methods to high-dimensional image data.
Further we show that it is better able to learn independent latent factors than each of its constituent components in isolation and than VAEs. 
For a preliminary demonstration of the inability of VAEs and the ability of Bijecta to discover ICA sources see Fig \ref{fig:dpsrites-demo}.

\section{Background}

\subsection{Independent Component Analysis}
\label{model}

The goal of ICA is to learn a set of statistically independent sources that `explain' our data.
ICA is a highly diverse modelling paradigm with numerous variants: learning a mapping vs learning a model, linear vs non-linear, different loss functions, different generative models, and a wide array of methods of inference \citep{Cardoso1989, Mackay1996, Lee2000}.

Here, we specify a generative model and find point-wise maximum likelihood estimates of model parameters in the manner of \citep{Mackay1996, Cardoso1997}. 
Concretely, we have a model with latent sources $\v{s} \in \mathcal{S} = \mathbb{R}^{d_s}$ generating data $\v{x} \in \mathcal{X} = \mathbb{R}^{d_x}$, with
$d_s \leq d_x$.
The linear ICA generative model factorises as
\begingroup
\setlength\abovedisplayskip{0pt}
\setlength\belowdisplayskip{0pt}
\[p(\v{x},\v{s})=p(\v{x}|\v{s})p(\v{s}), \qquad p(\v{s}) = \prod_{i=1}^{d_s} p(s_i),\]
\endgroup
where $p(\v{s})$ is a set of independent distributions appropriate for ICA. 
In linear ICA, where all mappings are simple matrix multiplications, 
the sources \textit{cannot} be Gaussian distributions.
Recall that we are mixing our sources to generate our data:
A linear mixing of Gaussian random variables is itself Gaussian, so unmixing is impossible
\citep{Lawrence2000}.
To be able to unmix, to break this symmetry, we can choose any heavy-tailed or light-tailed non-Gaussian distribution as our prior $p(\v{s})$
that gives us axis alignment and independence between sources.
A common choice is the family of generalised Gaussian distributions,
\begin{align}
p(s_i) &= \mathrm{GG}(s_i|\mu,\alpha,\rho) 
\nonumber \\
p(s_i) &= \frac{\rho}{2\alpha \Gamma(1/\rho)}\exp\left[{\left(-\frac{\lvert s_i - \mu\rvert}{\alpha}\right)^\rho} \right]
\label{eq:gnd}
\end{align}
with mean $\mu$, scale $\alpha$ and shape $\rho$.
For $\rho=2$ we recover the Normal distribution, and for $\rho=1$ we have the (heavy-tailed) Laplace. 
As $\rho \to \infty$ the distribution becomes increasingly sub-Gaussian, tending to a uniform distribution. 
As such, the generalised Gaussian is a flexible framework for specifying ICA-appropriate priors as it allows for the specification of a sub or super Gaussian distribution by way of a single parameter: $\rho$.

\subsection{Manifolds for the unmixing matrix $\v{A}^{+}$}
\label{sec:matrix-manifolds}

In linear ICA we want to find the linear mapping $\v{A}^{+}$ resulting in \textit{maximally independent} sources.
This is more onerous than merely finding decorrelated sources, as found by principal component analysis (PCA).

When learning a linear ICA model we typically have the mixing matrix $\v{A}$ as the (pseudo)inverse of the unmixing matrix $\v{A}^{+}$ and focus on the properties of $\v{A}^{+}$ to improve convergence. 
$\v{A}^{+}$ linearly maps from the data-space $\mathcal{X}$ to the source space $\mathcal{S}$.
It can be decomposed into two linear operations.
First we \textit{whiten} the data such that each component has unit variance and these components are mutually uncorrelated.
We then apply an orthogonal transformation and a scaling operation \nbcite[\S6.34]{Hyvarinen2001} to `rotate' the whitened data into a set of coordinates where the sources are independent \textit{and} decorrelated.
Whitening on its own is not sufficient for ICA --- 
two sources can be uncorrelated \textit{and} dependent 
(see Appendix \ref{app:correlationVSIndependence}).

Thus we can write the linear ICA unixing matrix as
\begin{equation}
\v{A}^+ = \v{\Phi R}\v{W}
\label{eq:eq:classic_ica_unmix}
\end{equation}
where $\v{W}\in\mathbb{R}^{d_s\times d_x}$ is our whitening matrix, $\v{R}\in\mathbb{R}^{d_s\times d_s}$ is an orthogonal matrix and $\v{\Phi}\in\mathbb{R}^{d_s}$ is a diagonal matrix.
Matrices that factorise this way are known as the \textit{decorrelating matrices} \citep{Everson1999}: members of this family decorrelate through $\v{W}$,
and $\v{\Phi R}$ ensures that  sources are statistically independent, not merely uncorrelated.
The optimal ICA unmixing matrix is the decorrelating matrix that decorrelates \textit{and} gives independence.

\subsection{Flows}
\label{flows}
Flows are models that stack numerous invertible changes of variables.
One specifies a simple base distribution and learns a sequence of (invertible) transforms to construct new distributions that assign high probability to observed data.
Given a variable $\v{z}\in\mathcal{Z}=\mathbb{R}^{d_x}$, we specify the distribution over data $\v{x}$ as
\begin{equation}
p(\v{x}) = p(\v{z})\norm{\mathrm{det}\frac{\partial f^{-1}}{\partial \v{z}}},
\label{eq:changeofvars}
\end{equation}
where $f$ is a bijection from $\mathcal{Z} \to \mathcal{X}$, ie $\mathbb{R}^{d_x} \to \mathbb{R}^{d_x}$, and $p(\v{z})$ is the base distribution over the latent $\v{z}$ \citep{JimenezRezende, deepflows}. 

For more flexible distributions for $\v{x}$, we
specify $\v{x}$ through a series of composed functions, from our simple initial $p$ into a more complex multi-modal distribution; for example for a series of $K+1$ mappings, \(\v{z} = f_K \circ ... \circ f_0(\v{x})\).
By the properties of determinants under function composition 
\begin{equation}
p(\v{x}) = p(\v{z}_K)\prod_{i=0}^{K}\norm{\mathrm{det}\frac{\partial f_{i}^{-1}}{\partial \v{z}_{i+1}}},
\label{eq:deepflow}
\end{equation}
where $\v{z_i}$ denotes the variable resulting from the transformation $f_{i}(\v{z}_{i})$, $p(\v{z}_K)$ defines a density on the $K^{\mathrm{th}}$, and the bottom most variable is our data ($\v{z}_0 = \v{x}$).

Computing the determinant of the Jacobian ($\mathrm{det}\frac{\partial f^{-1}}{\partial \v{z}}$) in Eq.~\eqref{eq:changeofvars} can be prohibitively costly, especially when composing multiple functions as in Eq.~\eqref{eq:deepflow}. 
To address this, flows use \textit{coupling layers} that enforce a lower triangular Jacobian such that the determinant of the Jacobian is simply the product of its diagonal elements.
We use recently proposed coupling layers based on rational quadratic splines (RQS) to enforce this lower triangular structure~\citep{Durkan2019}. 
They form highly flexible flows that typically require fewer composed mappings to achieve good performance relative to other coupling layers.
See Appendix \ref{app:coupling-layers} for details.

\section{Non-Square ICA using Flows}

Variational Autoencoders seem like a natural fit for learning a compressed set of statistically independent latent variables \citep{Kingma2013, Rezende2014}. 
It seems natural to train a VAE with an appropriate non-Gaussian prior, 
and expect that it would learn an appropriate ICA model.
However, this is not the case.
In \cite{Khemakhem2019} some experiments suggest that VAEs with ICA-appropriate priors are unsuited to performing non-linear ICA.
In our experiments (\S\ref{sec:exps}) we further verify this line of inquiry and show that VAEs struggle to match their aggregate posteriors to non-Gaussian priors and thus are unable to discover independent latent sources.

Though source separation can be achieved by `disentangling' methods such as the $\beta$-VAE \citep{Higgins2017} and $\beta$-TCVAE \citep{Chen2018},
these methods require post-hoc penalisation of certain terms of the VAE objective, at times inducing improper priors (in the $\beta$-TCVAE in particular \citep{Mathieu2019}).
Further, precise tuning of this penalisation, a form of soft supervision, is key to getting appropriate representations \citep{Rolinek2019, Locatello2019}.
\cite{Stuhmer2019} obtains a variety of non-linear ICA using VAEs with sets of Generalised Gaussian priors, but even then $\beta$ penalisation is required to obtain 'disentangled' representations.

As such there is a need for probabilistic DGMs that can separate sources without added hyperparameter tuning and that can do so by matching ICA-appropriate priors.
Our solution combines linear ICA with a dimensionality-preserving invertible flow $f_\theta$.
The flow acts between our data space of dimensionality and the representation fed to the linear ICA generative model; learning a representation that is well fit by the simple, linear ICA model. 
As we demonstrate in experiments (\S\ref{sec:exps}), this hybrid model, which we call Bijecta, succeeds where VAEs fail: it can match non-Gaussian priors and is able to discover independent latent sources on image datasets.

\subsection{A Linear ICA base distribution for flows}

Our aim here is to develop a non-square  ICA method that is both end-to-end differentiable \textit{and} computationally efficient, such that it can be trained jointly with a flow via stochastic gradient descent. 
We begin by choosing our base ICA source distribution to be a set of independent generalised Gaussian distributions, Eq~\eqref{eq:gnd} with $\mu=0$, $\alpha=1$ and $\rho$ varying per experiment; and the ICA model's likelihood to a be a Gaussian. 
\begin{align}
    &p(s_i) = \mathrm{GG}(s_i |\mu=0,\alpha=1,\rho), {\mathrm{for\ } i \in \{1,\dots,d_s}\}, \nonumber \\
    &p(\v{z}|\v{s}) = \mathcal{N}(\v{x}|\v{A}\v{s}, \v{\Sigma}_\theta), \nonumber
\end{align}
where $\v{A} \in \mathbb{R}^{d_x\times d_s}$ is our (unknown) ICA mixing matrix, which acts on the sources to produce a linear mixture; and 
$\v{\Sigma}_\theta$ is a learnt or fixed diagonal covariance.
This linear mixing of sources yields
an intermediate representation $\v{z}$ that is then mapped to the data by a flow. 
Our model has three sets of variables: the observed data $\v{x}$, the flow representation $\v{z} = f^{-1}(\v{x})$, and ICA latent sources $\v{s}$.
It can be factorised as
\begin{align}
    p_\theta(\v{x},\v{s})&=p_\theta(\v{x}|\v{s})p(\v{s})=p(\v{z}|\v{s})p(\v{s})\norm{\mathrm{det}\frac{\partial f_\theta^{-1}}{\partial \v{z}}}
\end{align}

While it is simple to train a flow by maximum likelihood method when we have a simple base distribution in $\mathcal{Z}$, here to obtain a maximum likelihood objective we would have to marginalise out $\v{s}$ to obtain the evidence in $\mathcal{Z}$; a computationally intractable procedure:
\begin{equation}
p(\v{z};\v{A}, \v{\Sigma}_\theta) = \int \mathrm{d}\v{s}\, p(\v{z}|\v{s};\v{A}, , \v{\Sigma}_\theta)p(\v{s}).
\label{eq:icaint}
\end{equation}
A contemporary approach is to use amortised variational inference for the linear ICA part of our model.
This means we introduce an approximate amortised posterior for $\v{s}$ and perform importance sampling on Eq (\ref{eq:icaint}), taking gradients through our samples using the reparameterisation trick \citep{Kingma2013, Rezende2014}.
Amortised stochastic variational inference offers numerous benefits:
it scales training to large datasets by using stochastic gradient descent, %
our trained model can be applied to new data with a simple forward pass,
and we are free to choose the functional \& probabilistic form of our approximate posterior.
Further our ICA model is end-to-end differentiable, making it optimal for jointly training with a flow. 

We choose a linear mapping in our posterior, with \(q_\phi(\v{s}|\v{z}) = \mathrm{Laplace}(\v{s}|\v{A}^{+}\v{z}, \v{b}_\phi)\),
where we have introduced variational parameters $\phi=\{\v{A}^{+}, \v{b}_\phi\}$ corresponding to an unmixing matrix and a diagonal diversity.
Using samples from this posterior we can define a lower bound $\ELBO$ on the evidence in $\mathcal{Z}$
\begin{align}
\log p(\v{z};\v{A},&\v{\Sigma}_\theta)  \geq \ELBO(\v{z};\phi,\v{A},\v{\Sigma}_\theta) \nonumber \\ &= \expect_{\v{s}\sim q}[\log p(\v{z}|\v{s}) -\KL(q_\phi(\v{s}|\v{z})|| p(\v{s}))
\label{eq:elbo_z_ica}
\end{align}
Using the change of variables equation, Eq (\ref{eq:changeofvars}), and the lower bound on the evidence for ICA in \eqref{eq:elbo_z_ica} for $\mathcal{Z}$, we can obtain a variational lower bound on the evidence for our data $\v{x}$ as the sum of the ICA model's ELBO (acting on $\v{z}$) and the log determinant of the flow:
\begin{align}
\log p_\theta(\v{x};\v{A}, &\v{\Sigma}_\theta)  \geq  \ELBO(\v{x};\theta, \phi,\v{A}, \v{\Sigma}_\theta) \nonumber \\ &= \ELBO(\v{z};\phi,\v{A},\v{\Sigma}_\theta)
+ \log \norm{\mathrm{det}\frac{\partial f_\theta^{-1}}{\partial \v{z}}}
\label{eq:total_var_objective}
\end{align}
As such our model is akin to a flow model, but with an additional latent variable $\v{s}$; the base distribution $p(\v{z})$ of the flow is defined through marginalizing out the linear mixing of the sources.
We refer to a model with $n$ non-linear splines mapping from $\mathcal{X}$ to $\mathcal{Z}$ as an $n$-layer Bijecta model.

\begin{figure}[t]
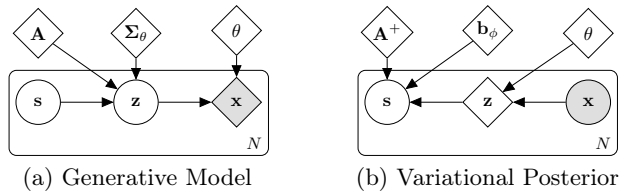

\centering
\noindent\makebox[0.5 \textwidth]{%
\subfloat[Generative Model]{\scalebox{0.7}{\raisebox{0ex}{
\tikz{ %
      \node[obs_det,minimum size=27pt] (x) {$\v{x}$} ; %
    \node[latent, left=of x,minimum size=24pt] (z) {$\v{z}$} ; %
    \node[latent, left=of z,minimum size=24pt] (s) {$\v{s}$} ; %
    \node[det, above=0.4cm of s,minimum size=27pt] (A) {$\v{A}$} ; %
    \node[det, above=0.4cm of z,minimum size=27pt] (b) {$\v{\Sigma}_\theta$} ; %
    \node[det, above=0.4cm of x,minimum size=27pt] (theta) {$\theta$} ; %
    \edge {s} {z} ; %
    \edge {A} {z} ; %
    \edge {z} {x} ; %
    \edge {theta} {x} ; %
    \edge {b} {z} ; %
    \plate {plate1} {(x) (z) (s)} {\scalebox{1}{{$N$}}}; %
  }}
  \label{fig:gout}
  }}
\hspace*{9mm}
\subfloat[Variational Posterior]{\scalebox{0.7}{
\tikz{ %
      \node[obs,minimum size=24pt] (x) {$\v{x}$} ; %
    \node[det, left=of x,minimum size=27pt] (z) {$\v{z}$} ; %
    \node[latent, left=of z,minimum size=24pt] (s) {$\v{s}$} ; %
    \node[det, above=0.4cm of s,minimum size=27pt] (A) {$\v{A}^{+}$} ; %
    \node[det, above=0.4cm of z,minimum size=27pt] (bphi) {$\v{b}_\phi$} ; %
    \node[det, above=0.4cm of x,minimum size=27pt] (theta) {$\theta$} ; %
    \edge {x} {z} ; %
    \edge {z} {s} ; %
    \edge {A} {s} ; %
    \edge {bphi} {s} ; %
    \edge {theta} {z} ; %
    \plate {plate1} {(x) (z) (s)} {\scalebox{1}{{$N$}}}; %
  }}
  \label{fig:gin_var}}
  }
 \caption{The generative model (a) and variational posterior (b), as defined in Eq (\ref{eq:total_var_objective}).}
 \label{fig:modelgraphs}
\end{figure}

In the case of non-square ICA, where our ICA model is not perfectly invertible, errors when reconstructing a mapping from $\mathcal{S}$ to $\mathcal{Z}$ may amplify when mapping back to $\mathcal{X}$. 
To mitigate this we add an additional regularisation term in our loss that penalises the $L_1$ error of each point when reconstructed into $\mathcal{X}$.
This penalisation can be weighted according to the importance of high-fidelity reconstructions for a given application. 

We attempted to train Bijecta with unconstrained %
mixing and unmixing matrices, but found that jointly training a linear model with a powerful flow was not trivial and models failed to converge when naively optimising Eq \eqref{eq:total_var_objective}. 
We found it crucial to appropriately constrain the unmixing matrix to get models to converge.
We detail these constraints in the next section.

\section{Whitening in $\v{A}^{+}$, without SVD}
\label{sec:nonsquareica}

\begin{figure*}
\centering
    \subfloat[Linear ICA]{\includegraphics[width=0.45\textwidth]{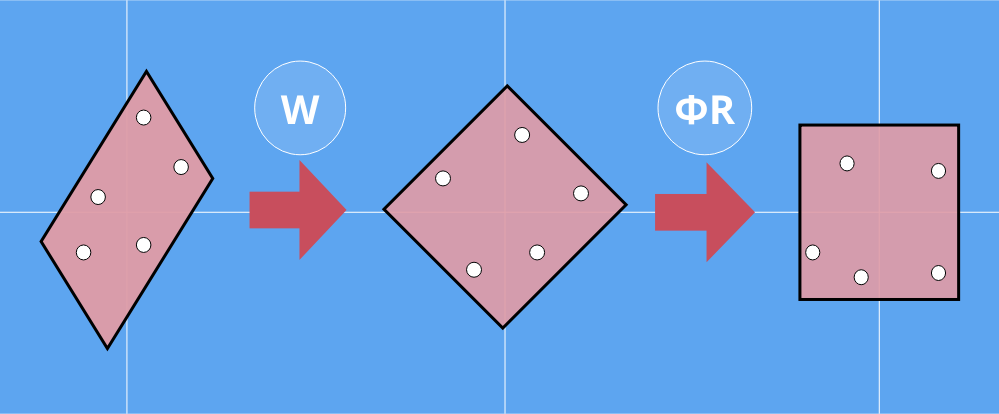}}
    \hspace{1cm}
    \subfloat[Bijecta]{
    \includegraphics[width=0.45\textwidth]{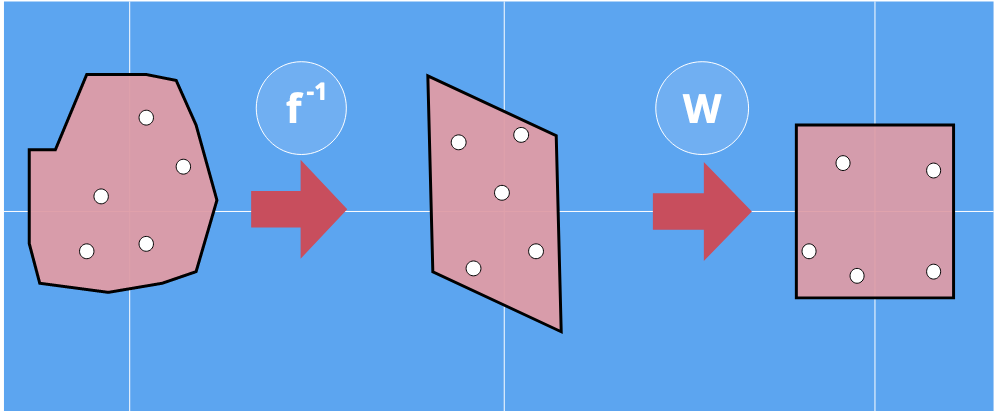}}
    \caption{
    (a) Sequence of actions that are performed by the elements of $\v{A}^+$, the unmixing matrix of linear ICA. $\v{W}$ whitens the correlated data and $\v{\Phi}\v{R}$ then ensures that the whitened (decorrelated) data is also independent.
    (b) Sequence of actions that are performed by the elements of Bijecta. $f^{-1}$ maps data to a representation for which the whitening matrix \textit{is} the ICA matrix. 
    $\v{W}$ now whitens $f^{-1}(\v{x})$ and the result is \textit{also} statistically independent. }
     \label{fig:BIJseq}
\end{figure*}

What are good choices for the mixing and unmixing matrices?
Recall in Sec~\ref{sec:matrix-manifolds} we discussed various traditional approaches to constraining the unmixing matrix.
For our flow-based model, design choices as to the parameterisation of $\v{A}^{+}$ stabilise and accelerate training.
As before, the mixing matrix $\v{A}$ is unconstrained during optimisation.
However, without the constraints on $\v{A}^{+}$ we describe in this section, we found that joint training of a flow with linear ICA did not converge.

Recall Eq~\eqref{eq:eq:classic_ica_unmix} --- linear ICA methods carry out whitening $\v{W}$, performing dimensionality reduction projecting from a $d_x$-dimensional space to a $d_s$-dimensional space, and the remaining rotation and scaling operations are square.
When training with a flow the powerful splines we are learning can fulfill the role of the square matrices $\v{R}$ and $\v{\Phi}$, but doing this ahead of the whitening itself.
Put another way, the outputs from the flow can be learnt such that they are simply a whitening operation away from being effective ICA representations in $\mathcal{S}$.
Thus, to minimise the complexities of jointly training a powerful flow with a small linear model, we can simply set  $\v{A}^+ = \v{W}$, such that the unmixing matrix projects from $d_z$ to $d_s$ and is decorrelating.
Statistical independence will come from the presence of the $\KL$ term in Eq~\eqref{eq:elbo_z_ica}: the flow will learn to give $\v{z}$ representations that, when whitened, are good ICA representations in $\mathcal{S}$.
See Fig \ref{fig:BIJseq} for a visual illustration of this process and a comparison with the steps involved in linear ICA. 

In previous linear ICA methods, the whitening procedure $\v{W}$ has been derived in some data-aware way. 
A common choice is to whiten via the Singular Value Decomposition (SVD) of the data matrix, where $\v{W} = \v{\Sigma}\v{U}^T$, $\v{\Sigma}$ is the rectangular diagonal matrix of singular values of $\v{X}$, and the columns of $\v{U}$ are the left-singular vectors.
Computing the SVD of the whole dataset is expensive for large datasets; for us, in the context of Bijecta, we would be re-calculating the SVD of the representations $\v{Z}=f^{-1}(\v{X})$ of the entire dataset after every training step. 
One route around this would be online calculation of the whitening matrix \citep{Cardoso1996, Hyvarinen2001}.
This introduces an extra optimisation process that also has to be tuned, and would interact with the training of the flow.

To tackle these shortcomings of existing whitening methods, we propose a new method for linear non-square ICA that uses Johnson–Lindenstrauss (JL) transforms (also known as sketching) \citep{Woodruff2014}, which not only works effectively as a linear ICA method, but also works in conjunction with a flow model.  
These JL transforms have favourable properties for ICA, as we demonstrate in theoretical results. 
Further, this method samples part of the whitening matrix at initialisation and leaves it fixed for the remainder of training, requiring \textit{no hyper-parameter tuning} and making it extremely computationally efficient. 
This method is novel and efficient when used as a whitening method within linear ICA, and when combined with a flow as in Bijecta is a powerful method for non-linear ICA as we demonstrate in experiments.

\subsection{Approximately-Stiefel matrices}

We have set  $\v{A}^+ = \v{W}$, the whitening matrix. 
$\v{W}$ has two aims in non-square ICA. The first is dimensionality reduction, projecting from a $d_x$-dimensional space to a $d_s$-dimensional space. 
The second is to decorrelate the data it transforms, meaning that the resulting projection will have unit variance and mutually uncorrelated components.  
More formally we wish for $\v{W}$ of dimensionality $d_s \times d_x$ to be decorrelating. 

The set of orthogonal decorrelating rectangular matrices lie on the Stiefel Manifold \citep{Stiefel1935} denoted $\mathcal{V}$.
For matrices with $r$ rows and $c$ columns, a matrix $\v{G} \in \mathcal{V}(r,c)$ iff $\v{G}\v{G}^* = \v{I}$ ($\v{G}^*$ the conjugate transpose of $\v{G}$).
Constraining the optimisation of $\v{W}$ to this manifold can be computationally expensive and complex \citep{Bakir2004, Harandi2016, Siegel2019}
and instead we choose for $\v{W}$ to be \textit{approximately} Stiefel, that is to lie close to $\mathcal{V}(d_s,d_x)$.
This is justified by the following theorem, proved in Appendix A:
\begin{theorem}
Let $\v{G}$ be a rectangular matrix and  $\tilde{\v{G}}$ be its projection onto $\mathcal{V}(r,c)$. 
As the Frobenius norm $||\v{G} - \tilde{\v{G}}|| \to 0$ we have that $||\v{G}\v{X}\v{X}^T\v{G}^T - \v{\Psi}|| \to 0$, where $\v{G}\v{X}\v{X}^T\v{G}^T$ is the cross-correlation of the projection of data $\v{X}$ by $\v{G}$, and $\v{\Psi}$ is some diagonal matrix. 
\label{th:1}
\end{theorem}
Simply put, this shows that as a matrix $\v{G}$ approaches the Stiefel manifold $\mathcal{V}(r,c)$ the off-diagonal elements of the cross-correlation matrix of the projection $\v{G}\v{X}$ are ever smaller, so $\v{G}$ is ever more decorrelating. 
Given these properties we want our whitening matrix to lie close to the Stiefel manifold.

\subsubsection{Johnson-Lindenstrauss projections} 

\begin{figure*}[t]
\centering
\hspace{0mm}
\subfloat[Source Images]{

\begin{tabular}[Source Images]{c}
\includegraphics[width=0.08\textwidth]{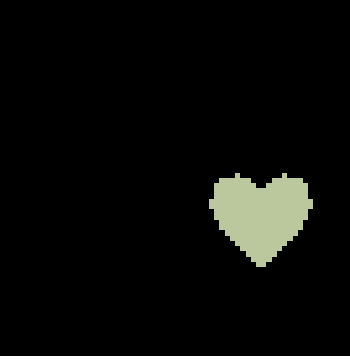}
\includegraphics[width=0.08\textwidth]{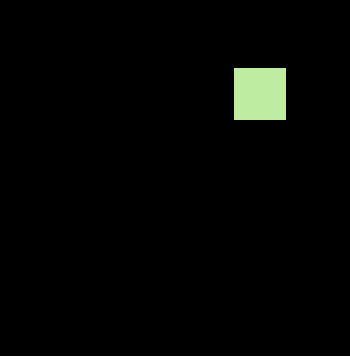} 
\\
\end{tabular}}
\hspace{0mm}
\subfloat[Mixed]{
\begin{tabular}[$\v{A}$ from FastICA]{c}
\includegraphics[width=0.08\textwidth]{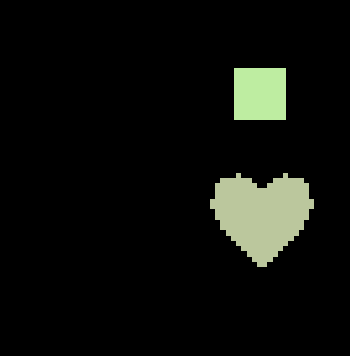}
\end{tabular}}
\hspace{0mm}
\subfloat[$\v{A}$ FastICA]{
\begin{tabular}[$\v{A}$ from FastICA]{c}
\includegraphics[width=0.08\textwidth]{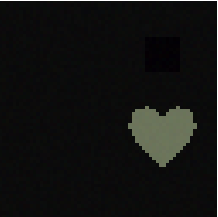}
\includegraphics[width=0.08\textwidth]{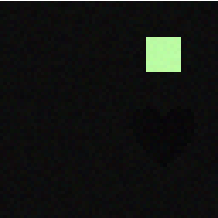} 
\end{tabular}}
\hspace{0mm}
\subfloat[$\v{A}$ JL-Cayley ICA]{
\begin{tabular}[$\v{A}$ from JL-Cayley ICA]{c}
\includegraphics[width=0.08\textwidth]{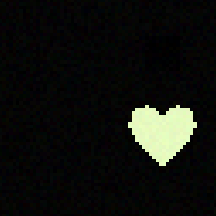}
\includegraphics[width=0.08\textwidth]{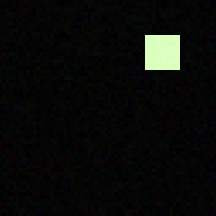} 
\end{tabular}} \\

\caption{Here we run linear ICA on a pair of images (a) that are mixed linearly ($\mathrm{mix} = w_1*\mathrm{image}_1 + w_2*\mathrm{image}_2$) (b) to form a dataset with 512 points.
In both cases $w_1$ and $w_2$ are sampled from a uniform distribution.
We plot the mixing matrix $\v{A}$ for our JL-Cayley model with a quasi-uniform $\GG$ prior with $\rho=10$ (c) and for FastICA  \citep{Hyvarinen1997} as a benchmark.
$\v{A}$ should recover the source images, which occurs for both models.
}
\vspace{-0.5cm}
\label{fig:linear_experiments}
\end{figure*}

By Theorem 1 we know that we want our whitening matrix to be close to $\mathcal{V}(d_s,d_x)$. 
How might we enforce this closeness? 
By the definition of the Stiefel manifold, we can intuit that a matrix $\v{G}$ will lie close to this manifold if  $\v{GG^T}\approx \v{I}$. 
We formalise this as:

\begin{theorem}

Let $\v{G}\in \mathbb{R}^{d_s\times d_x}$ and let $\tilde{\v{G}}$ be its projection onto $\mathcal{V}(d_s,d_x)$.  
As the Frobenius norm $\|\v{G}\v{G}^{T}-\v{I}\| \to 0$, we also have $\|\tilde{\v{G}}-{\v{G}}\| \to 0$. 
\label{th:2}
\end{theorem}

The proof for this is presented in Appendix \ref{app:close}. Using this theorem, we now propose an alternative to SVD-based whitening. 
Instead of having $\v{W}=\v{\Sigma}^{-1}\v{U}^T$ be the result of SVD on the data matrix, we define our whitening matrix as a data-independent Johnson–Lindenstrauss transform.
We must ensure that $\v{W}$, our rectangular matrix, is approximately orthogonal, lying close to the manifold $\mathcal{V}(d_s,d_x)$.
More formally by Theorem \ref{th:2}, our goal is to construct a rectangular matrix $\v{W}$ such that $\v{W}\v{W}^T \approx \v{I}$.

We construct approximately orthogonal matrices for $\v{W}$ by way of Johnson-Lindenstrauss (JL) Projections \citep{Johnson1984}.
A JL projection $\v{W}$ for $\mathbb{R}^{d_x} \rightarrow \mathbb{R}^{d_s}$ is sampled at initialisation from a simple binary distribution \cite{Achlioptas2003}:
\begin{align}
W_{i,j} = \begin{cases}
      +1/\sqrt{d_s}, & \text{with probability } \frac{1}{2}  \\
      -1/\sqrt{d_s}, & \text{with probability } \frac{1}{2}
    \end{cases}
\label{eq:jlbin1}
\end{align}
This distribution satisfies $\mathbb{E}[\v{W}\v{W}^{T}] = \v{I}$, and such a draw has $\v{W}\v{W}^T \approx \v{I}$. 
We choose to fix $\v{W}$ after initialisation such that $\v{A}^+=\v{W}$ never updates, greatly simplifying optimisation.

\section{Experiments}
\label{sec:exps}
\begin{figure*}[t!]
   \centering
    
    \subfloat[Sample Images]{

\begin{tabular}[Samples]{c}
\includegraphics[width=0.06\textwidth]{figures/dsprites/multi/sample1.png}
\includegraphics[width=0.06\textwidth]{figures/dsprites/multi/sample2.png}
\\
\end{tabular}}
\hspace{1cm}
\subfloat[VAE Latent Traversals]{
\begin{tabular}[Samples]{c}
\includegraphics[width=0.25\textwidth]{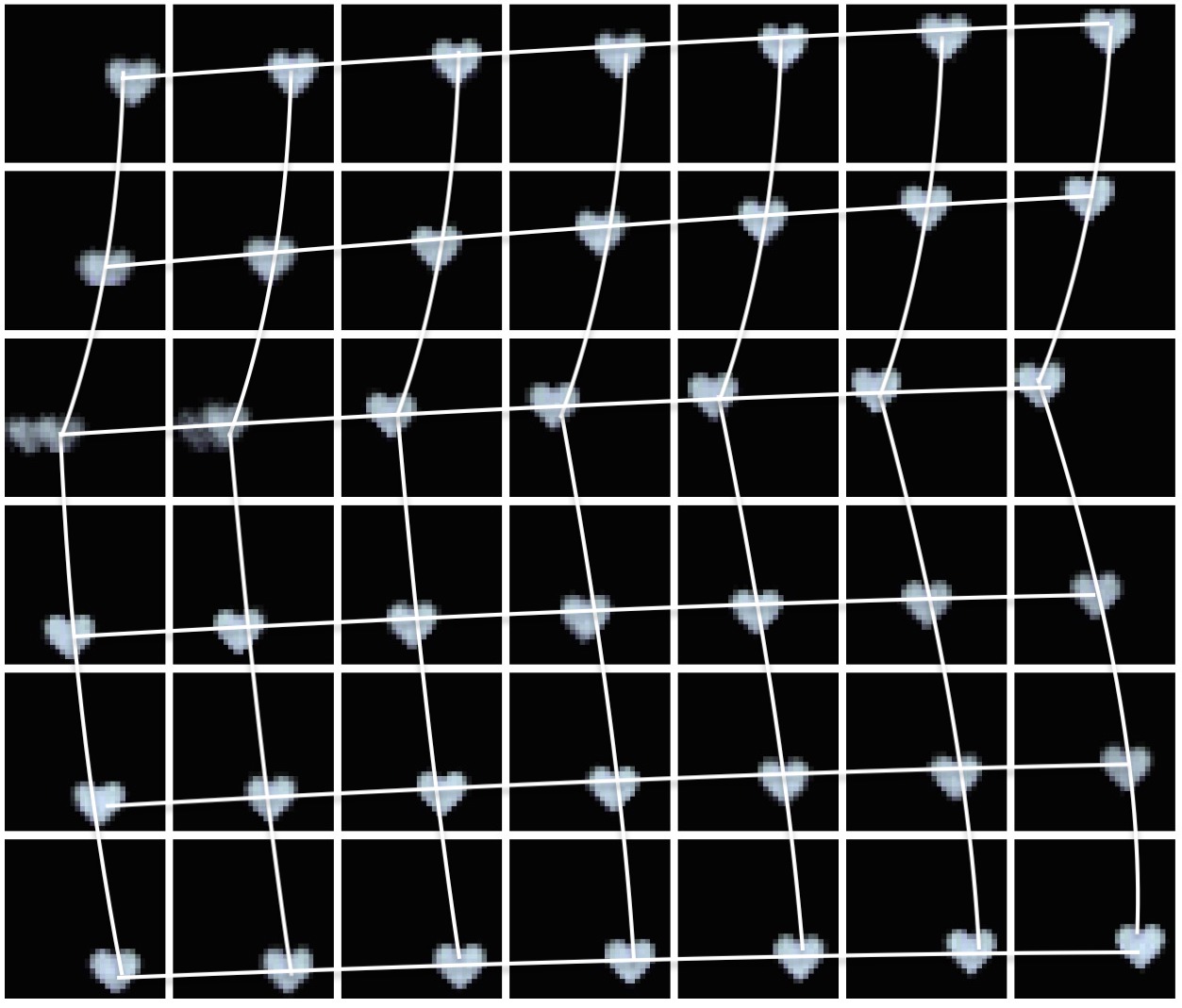}
\end{tabular}}
\hspace{1cm}
\subfloat[Bijecta Latent Traversals]{
\begin{tabular}[Samples]{c}
\includegraphics[width=0.25\textwidth]{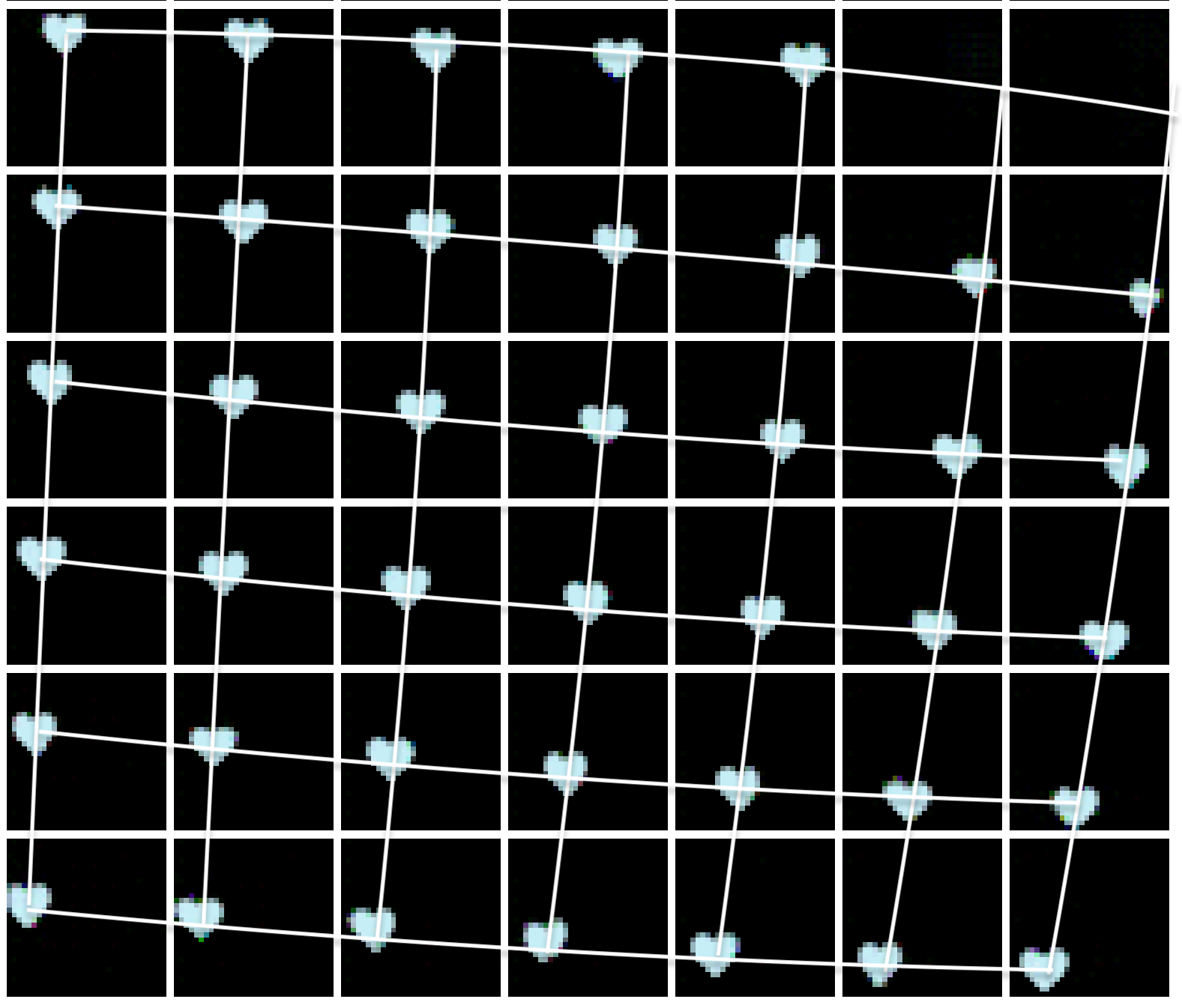}
\end{tabular}}

\hspace{0mm}

    \caption{
    Here we demonstrate that Bijecta is capable of unmixing non-linearly mixed sources, better than VAEs with ICA-appropriate priors. 
    We take a dSprites heart and, using a randomly sampled affine transformation, move it around a 32 by 32 background (a). 
    With 2-D $\GG$ priors with $\rho=10$ for a convolutional VAE (b) and for Bijecta (c) we plot the generations resulting from traversing the 2-D latent-source space in a square around the origin.
    We sketch the learnt axis of movement of the sprite with white lines.
    In (b) the VAE does not ascribe consistent meaning to its latent dimensions.
    It has failed to discover consistent independent latent sources: it has a sudden change in the learnt axes of movement along the second dimension, as seen by the kink in the white vertical lines.  
    In (c) Bijecta is able to learn a simple affine transformation along each latent dimension, consistently spanning the space.
    In Fig \ref{fig:dsprites-affine-latents} we show the posterior distributions of both these models and show that Bijecta is better able to match the GG prior than the VAE, supporting our findings here. 
    }
    \label{fig:dpsrites-affine}
    \vspace{-0.2cm}
\end{figure*}

Here we show that our approach outperforms VAEs and flows with ICA-priors at discovering ICA sources in image data. 
But first, as a sanity check, we show that a linear ICA model using JL projections to whiten can successfully unmix linearly mixed sources in Fig \ref{fig:linear_experiments}.
For details on how to implement such a linear ICA model, see Appendix \ref{app:linear_ica}. 
We take a pair of images from dSprites and create linear mixtures of them.
We see that linear ICA with JL projections can successfully discover the true sources, the images used to create the mixtures, in the columns of $\v{A}$.

\begin{figure*}[h!]
\centering
\subfloat[Bijecta Factorised Posterior]{\includegraphics[height=3.0cm]{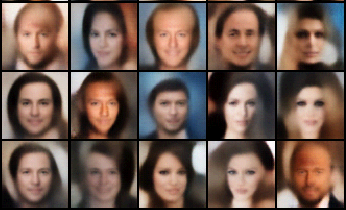}
\label{fig:marginalpostsamples}}
\hspace{0.0cm}
 \raisebox{0.03cm}{
\subfloat[Bijecta latent 5 traversal]{\begin{tabular}[b]{c}%
\vspace{-0.15cm}
\adjincludegraphics[height=0.995cm,trim={0 0 0 {0.01\height}},clip]{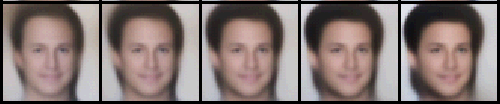}\\
\vspace{-0.15cm}
\adjincludegraphics[height=1.0cm,trim={0 0 {0.02\height} {0.01\height}},clip]{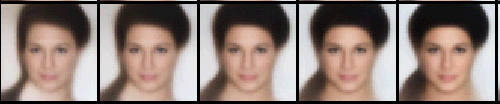} \\
\vspace{-0.14cm}
\adjincludegraphics[height=0.993cm,trim={0 {0.02\height} {0.02\height} 0 },clip]{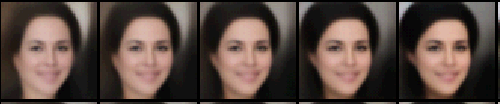}
\end{tabular}}}
\hspace{0.0cm}
 \subfloat[RQS flow traversal]{\includegraphics[height=3.0cm]{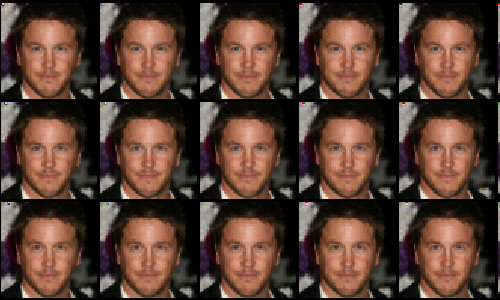}}
     \caption{
     (a) shows decodings from an 8-layer Bijecta ($d_s=32$) trained on CelebA with a Laplace prior (GG $\rho=1$) where we sample from the factorised approximation to Bijecta's posterior.
     See Fig \ref{fig:post-samples-celeba} for more such samples. 
     (b) shows latent traversals for 3 different datapoints all along the same axis-aligned direction, for this same model.
     (c) shows traversals for a single embedded training datapoint from CelebA moving along 3 latent directions in an RQS flow with Laplace base distribution. 
     Though we have selected 3 dimensions, all $\mathcal{Z}$ dimensions had similar latent traversals.
     In (b-c) Images in the center correspond to the original latent space embedding, on either side we move up to 6 standard deviations away along this direction with other dimensions remaining fixed.
     The flow has not discovered axis-aligned transforms, whereas Bijecta has learned informative latent dimensions:
     here the dimension encodes hair thickness. 
     Note that identity is maintained throughout and that the transform is consistent across different posterior samples.
     See Appendix \ref{app:traversals} for gallery of transforms for Bijecta. 
    }
     \label{fig:latent_traversal}
     \vspace{-0.1cm}
\end{figure*}

\paragraph{Affine Data}
Given that we have established that our novel theory for decorrelating matrices can produce standalone linear ICA models, we now want to ascertain that our hybrid model performs well in non-linear mixing situations. 
To do so we create a dataset consisting of a subset of dSprites where we have a light-blue heart randomly uniformly placed on a black field.
The true latent sources behind these randomly sampled affine transformations are simply the coordinates of the heart. 
First, in Fig  \ref{fig:dsprites-affine-latents} we demonstrate that linear ICA  models are unable to uncover the true latent sources.
As expected non-linear mixing regimes motivate the use of flexible non-linear models.

We now demonstrate that Bijecta can uncover the latent sources underpinning these affine transformations, whereas VAEs with ICA-appropriate priors fail to do so. 
For details of VAE architecture, see Appendix \ref{app:arch}.
These VAEs are able to learn to reconstruct data well, but the learnt latent space does not correspond to the underlying statistically independent sources (see Figs \ref{fig:dpsrites-demo} and \ref{fig:dsprites-affine-latents}). 
In fact for VAEs the effect of the latent variables is not consistent throughout the latent space, as seen in Fig \ref{fig:dpsrites-affine}. 
For Bijecta, the learnt latent space corresponds to the underlying statistically independent sources (see Figs \ref{fig:dpsrites-demo} and \ref{fig:dsprites-affine-latents}), and the meaning of the latent variables is consistent in Fig \ref{fig:dpsrites-affine}.
Further in Fig \ref{fig:dpsrites-affine} the model seems able to extrapolate outside the training domain: it generates images where the heart is partially rendered at the edges of the frame, even removing the heart entirely at times, even though such images are not in the training set.

\paragraph{Natural Images}
The previous experiments show that our model is capable of isolating independent sources on toy data. 
We complement this finding with experiments on a more complex natural image dataset, CelebA, and show that here too our model outperforms VAEs in learning factorisable representations. 

An ersatz test of this can be done by synthesising images where we sample from a factorised approximation of Bijecta's posterior. 
If the learned latent sources are actually independent, then the posterior over latent sources given the entire dataset should factorize into a product across dimensions, i.e. $q(\v{s}) = \prod_i q(\v{s}_i)$.
In this case, we can fit an approximation to the posterior by fitting $d_s$ independent one-dimensional density estimates on $q(\v{s}_i)$. If the sources are not independent, then this factorized approximation to the posterior will be missing important correlations and dependencies.
In Fig~\ref{fig:marginalpostsamples} samples from this factorised approximation look reasonable, suggesting that Bijecta has learnt representations that are statistically independent.

\begin{figure*}[t]
\centering
\includegraphics[width=0.45\textwidth]{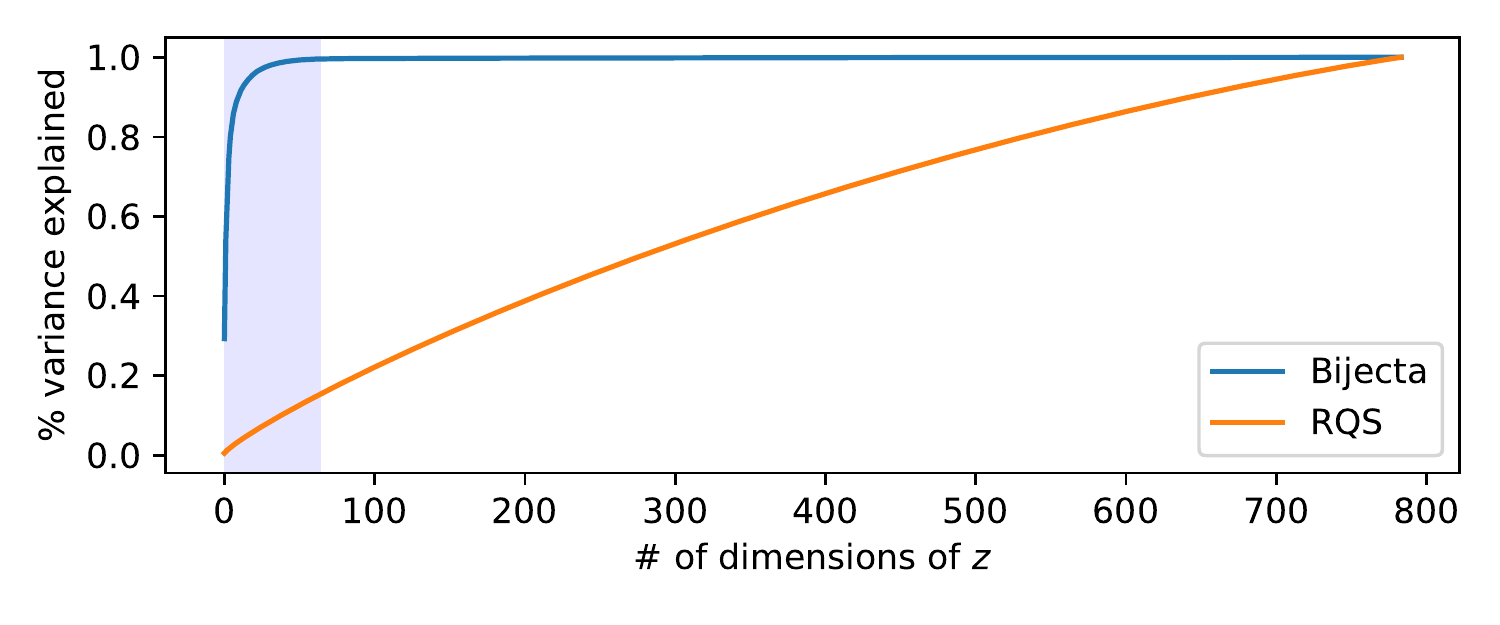}
\includegraphics[width=0.45\textwidth]{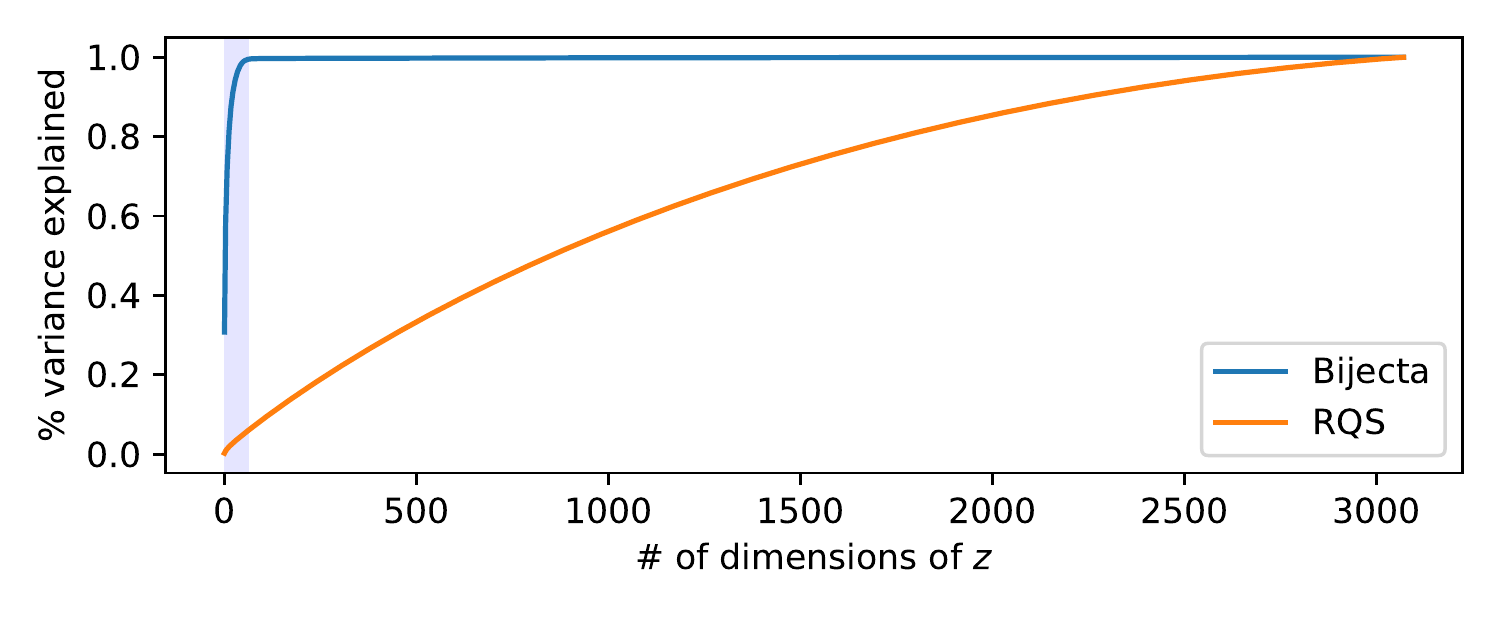}
\caption{Explained variance plots for the embedding in $\mathcal{Z}$, as measured by the sums of the eigenvalues of the covariance matrix of the embeddings, for both our Bijecta model and for an RQS model of equivalent size trained with a Laplace base distribution (GG distribution with $\rho=1$). For both Fashion-MNIST (left) and CIFAR 10 (right) datasets we see that the Bijecta model has learned a compressive flow, where most of the variance can be explained by only a few linear projections. The shaded region denotes the first 64 dimensions, corresponding to the size of the target source embedding $\mathcal{S}$.}
\vspace{-0.3cm}
\label{fig:variance_explained}
\end{figure*}
To quantify this source-separation numerically, we measure the total correlation (TC) of the aggregate posteriors of Bijecta ($q(\v{s}|\v{z})$) and VAEs ($q(\v{z}|\v{x})$) as \cite{Chen2018} do.
Intuitively, the TC measures how well a distribution is approximated by the product of its marginals -- or how much information is shared between variables due to dependence \citep{Watanabe1960}.
It directly measures how well an ICA model has learnt decorrelated and independent latent representations \citep{Everson2001}. 
Formally, it is the $\mathrm{KL}$ divergence between a distribution $r(\cdot)$ and a factorised representation of the distribution: \(\mathrm{TC} = \mathrm{KL}(r(\v{s})||\prod_i r(s_i)),\) where $i$ indexes over (latent) dimensions.

In Table \ref{table:tc-res}, we show that Bijecta learns an aggregate posterior with significantly lower TC values than both VAEs with Laplace priors, \textit{and} $\beta$-TCVAEs -- which in their training objective penalise the TC by a factor $\beta$ \citep{Chen2018}.
Our model has learnt a better ICA solution.
We also include numerical results in Appendix \ref{app:num_results} showing that Bijecta outperforms linear ICA on a variety of natural image datasets. 

\begin{table}[h]
  \caption{Total Correlation Results: We evaluate the source separation of different models on CelebA via the TC of the validation set embeddings in the 32-D latent space of: Laplace prior VAEs, $\beta$-TCVAEs ($\beta=15$), and Bijecta with a Laplace prior ($\pm$ indicates the standard deviation over 2 runs). VAEs use the same architecture and training as \cite{Chen2018}.}
  
  \label{table:tc-res}
  \centering
  \resizebox{\columnwidth}{!}{%
  \begin{tabular}{rccc}
    \toprule
                            & Laplace-VAE & $\beta$-TCVAE & Laplace-Bijecta\\
        \midrule
                                TC: & $106.7 \pm 0.9$  & $55.7 \pm 0.1$ & $\mathbf{13.1 \pm 0.4}$ \\
    \bottomrule
  \end{tabular}}
\end{table}

\paragraph{Dimensionality reduction on flow models}
To conclude, having shown that Bijecta outperforms VAEs on a variety of non-linear ICA tasks, we now contrast our model's ability to automatically uncover sources relative to flow models with heavy-tailed base distributions.
We do so by measuring the cumulative explained variance by the dimensions in $\mathcal{Z}$ for both models.
If a small number of dimensions explains most of $\mathcal{Z}$'s variance then the model has learnt a bijection which only requires a small number of dimensions to be invertible. 
It has in effect learnt the generating sources underpinning the data. 

In Fig~\ref{fig:variance_explained} we show that Bijecta induces better-compressed representations in $\mathcal{Z}$ than non-compressive flows.
We plot the eigenvalues of the covariance matrix on the output of the flow, i.e.\ on $\mathrm{Cov}(f(\v{X}))$, to see how much of the total variance in the learned feature space $\mathcal{Z}$ can be explained in a few dimensions.
In doing so we see that a flow trained jointly with a linear ICA model with $d_s = 64$ effectively concentrates variation into a small number of intrinsic dimensions; this is in stark contrast with the RQS flows trained with only a Laplace base distribution.
This demonstrates that our model is able to automatically detect relevant directions on a low dimensional manifold in $\mathcal{Z}$, and that the bijective component of our model is better able to isolate latent sources than a standard flow.

For a visual illustration of this source separation we show the difference in generated images resulting from smoothly varying along each dimension in $\mathcal{S}$ for Bijecta models and in $\mathcal{Z}$ for flows in Fig \ref{fig:latent_traversal}.
Bijecta is clearly able to discover latent sources, whereby it learns axis-aligned transformations of CelebA faces, whereas a flow with equivalent computational budget and a heavy-tailed base distribution is not able to.

All flow-based baselines are trained using the objective in Eq (\ref{eq:deepflow}), using Real-NVP style factoring-out \citep{Durkan2019,Dinh2015}, and are matched in size and neural network architectures to the flows of Bijecta models.
See Appendix \ref{app:arch} for more details.

\section{Related Work}
\label{related_work}
One approach to extend ICA to non-linear settings is to have a non-linear mapping acting on the independent sources and data \citep{Burel1992, Deco1995, Yang1998, Valpola2003}.
In general, non-linear ICA models have been shown to be hard to train, having problems of unidentifiability:
the model has numerous local minima it can reach under its training objective, each with potentially different learnt sources~\citep{Hyvarinen1999, Karhunen2001, Almeida2003, Hyvarinen2019}.
Some non-linear ICA models have been specified with additional structure to reduce the space of potential solutions, such as putting priors on variables \citep{Lappalainen2000} or specifying the precise non-linear functions involved \citep{Lee1997a, Taleb2002}, 
Recent work shows that conditioning the source distributions on some always-observed side information, say time index, can be sufficient to induce identifiability in non-linear ICA \citep{Khemakhem2019}.

Modern flows were first proposed as an approach to non-linear square ICA \citep{Dinh2015}, but are also motivated by desires for more expressive priors and posteriors \citep{Kingma, deepflows}.
Early approaches, known as symplectic maps \citep{Deco1995, Parra1995, Parra1996}, were also proposed for use with ICA.
Flows offer expressive dimensionality-preserving (and sometimes volume-preserving) bijective mappings \citep{Dinh2017, Kingma2018}.
Flows have been used to provide feature extraction for linear discriminative models \citep{Nalisnick2019}.
Orthogonal transforms have been used in normalizing flows before, to improve the optimisation properties of Sylvester flows \citep{VanDenBerg2018, Golinski2019}.
Researchers have also looked at constraining neural network weights to the Stiefel-manifold \citep{Li2020}.

\section{Conclusion}
We have developed a method for performing non-linear ICA large high-dimensional image datasets which combines state-of-the-art flow-based models and a novel theoretically grounded linear ICA method.
This model succeeds where existing probabilistic deep generative models fail: 
its constituent flow is able to learn a representation, lying in a low dimensional manifold in $\mathcal{Z}$, under which sources are separable by linear unmixing. 
In source space $\mathcal{S}$, this model learns a low dimensional, explanatory set of statistically independent latent sources.

\nocite{Absil2012}
\nocite{Kingma2015}
\clearpage
\newpage
\acknowledgments{
This research was directly funded by the Alan Turing Institute under Engineering and Physical Sciences Research Council (EPSRC) grant EP/N510129/1.
AC was supported by an EPSRC Studentship.
MW was supported by EPSRC grant EP/G03706X/1.
CH was supported by the Medical Research Council, the Engineering and Physical Sciences Research Council, Health Data Research UK, and the Li Ka Shing Foundation
SR gratefully acknowledges support from the UK Royal Academy of Engineering and the Oxford-Man Institute.

We thank Tomas Lazauskas, Jim Madge and Oscar Giles from the Alan Turing Institute's Research Engineering team for their help and support.}
\bibliographystyle{apalike2}
\bibliography{references_mend}

\newpage

\appendix
\setcounter{equation}{0}
\renewcommand\thefigure{\thesection.\arabic{figure}}
\setcounter{figure}{0}
 \onecolumn
\clearpage
  \hsize\textwidth
  \linewidth\hsize \toptitlebar {\centering
  {\Large\bfseries Appendix for Learning Bijective Feature Maps for Linear ICA \par}}
 \bottomtitlebar \vskip 0.1in
 \thispagestyle{appendixpage}

\section{Correlated and Dependent Sources}
\label{app:correlationVSIndependence}

\begin{figure}[h]
\begin{minipage}[c]{0.3\textwidth}

    \begin{center}
    \includegraphics[width=0.7\textwidth]{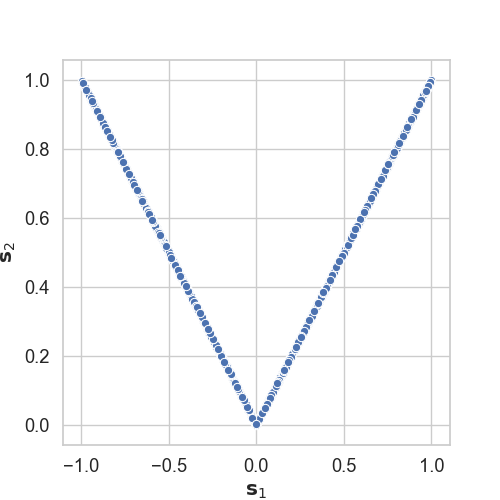}
    \end{center}
\end{minipage}
\begin{minipage}[c]{0.48\textwidth}
     \caption{Sources can be uncorrelated \textit{and} dependent. Consider our first source $\v{s}_1$ to be uniformly distributed on the interval $[-1, 1]$. If $\v{s}_1 \leq 0$, then $\v{s}_2=-\v{s}_1$, else $\v{s}_2=\v{s}_1$. In this case the variables are uncorrelated, $\expect[\v{s}_1\v{s}_2]=0$, but the joint distribution of $\v{s}_1$ and $\v{s}_2$ is not uniform on the rectangle $[-1, 1]\times[0, 1]$, as it would be if they were independent. See plot to the left for an illustration of this.}
     \label{fig:correlationVSIndependence}
\end{minipage}
\end{figure}

\section{VAEs underperform on ICA}
\label{app:post_plots}

\begin{figure*}[h]
    \centering
     \subfloat[][Linear ICA]{\includegraphics[width=0.23\textwidth]{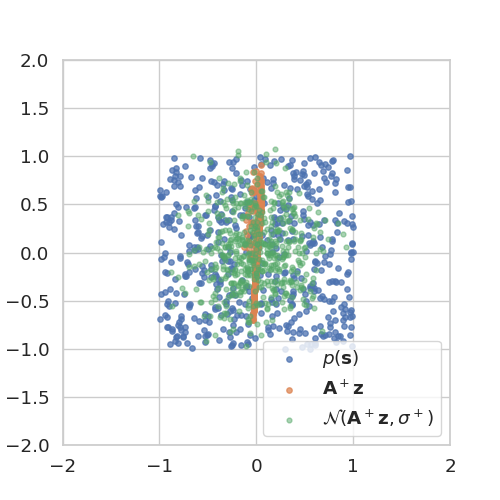}}
      \hspace{0.5cm}
    \subfloat[][Bijecta]{\includegraphics[width=0.23\textwidth]{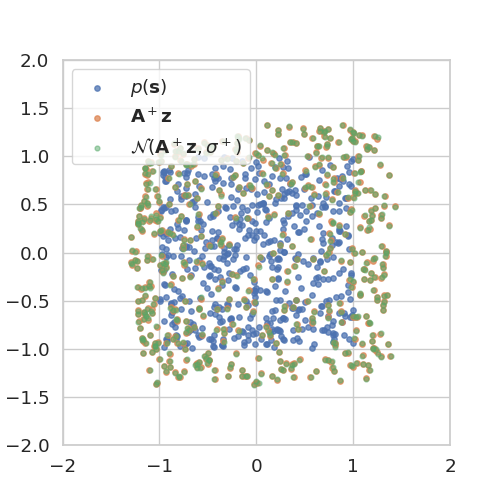}} 
    \hspace{0.5cm}
    \subfloat[][VAE]{\includegraphics[width=0.23\textwidth]{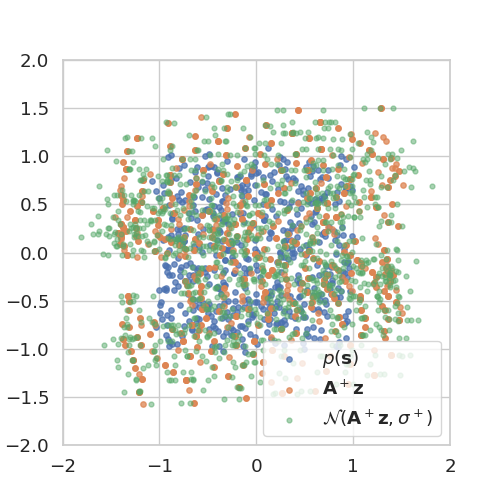}} 
    \caption{In (a), (b) we run linear ICA and a single-layer Bijecta on the affine transformation dataset of Fig \ref{fig:dpsrites-affine}.
    We take a dSprites heart and using a randomly sampled affine transformation, move it around a 32 by 32 background.
    We plot the posterior distribution $\mathcal{N}(\v{A}^+\v{z},\v{\sigma}^+)$ (green) and its mean $\v{A}^+\v{z}$ (orange) for both models. Clearly the posterior from Bijecta is better able to match the quasi-uniform $\GG$ prior with $\rho=10$ ($p(\v{s})$ in blue) than the linear ICA model, highlighting that the addition of the flow allows for linear unmixing. 
    Similarly the large VAE of Fig \ref{fig:dpsrites-affine} is unable to match this prior, forming two posterior `lobes' around the prior.}
    \label{fig:dsprites-affine-latents}
\end{figure*}

\newpage
\section{Proof of optimality}
\label{app:optimal}

\textit{Definition: We say a matrix $\v{B}'$ is strictly more orthogonal than a matrix $\v{B}$ if $\|\v{B'}^T\v{B'} - \v{I}\| < \|\v{B}^T\v{B} - \v{I}\|$}. \\

\setcounter{theorem}{0}

\begin{theorem}
As the Frobenius norm $\|\v{G} - \tilde{\v{G}}\| \to 0$, 
where $\v{G} \in \mathbb{R}^{r \times c}$ and $\tilde{\v{G}}$ is the projection of $\v{G}$ onto $\mathcal{V}(r,c)$ , $\|\v{G}\v{X}\v{X}^T\v{G}^T - \v{D}\| \to 0$, where $\v{G}\v{X}\v{X}^T\v{G}^T$ is the cross-correlation of the projection of data $\v{X}$ by $\v{G}$, and $\v{D}$ is some diagonal matrix. 

\end{theorem}

\begin{proof}

The Stiefel manifold is given by (assuming $r < c$ ): 

\begin{equation}
   \mathcal{V}(r,c) = \{\v{G} \in  \mathbb{R}^{r\times c}: \v{G}\v{G}^T=\v{I}\}
   \label{eq:stiefel}
\end{equation}

The unique projection $\tilde{\v{G}}$ onto this manifold of a matrix $\v{G} \in \mathbb{R}^{d_s \times d_x}$, with polar decomposition  $\v{U}\v{P}$, is simply $\v{U}$ \citep{Absil2012}.
$\v{P}$ denotes a $c \times c$  positive-semidefinite Hermitian matrix and $\v{U}$ is a $r \times c$ orthogonal matrix, i.e $\v{U} \in \mathcal{V}(r,c)$ and has a conjugate transpose denoted $\v{U}^*$ such that $\v{U}\v{U}^*=\v{I}$.
As such given any matrix $\v{G}$ we have a polar decomposition $\v{G}=\tilde{\v{G}}\v{P}$, where $\tilde{\v{G}}$, a linear decorrelating matrix, is the projection onto $\mathcal{V}(r,c)$ and $\v{P}$ denotes a $r\times c$  positive-semidefinite Hermitian matrix.

Let  $\v{G} \in \mathbb{R}^{r \times c}$ be some projection matrix of data $\v{X}$. We have $\v{S}= \v{G}\v{X}$.
The cross-correlation is expressed as $\v{S}\v{S^T} =\v{G}\v{X}\v{X}^T\v{G}^T$. 
In the case where $\v{S}$ is perfectly decorrelated, we have: 
$\v{S}\v{S^T} = \v{D}^2$ where $\v{D}$ is a diagonal matrix. 
We know that the Stiefel manifold $\mathcal{V}(r,c)$ (defined in Eq (\ref{eq:stiefel})) holds the set of all whitening matrices \citep{Everson1999}, up to the diagonal matrix $\v{D}^{-1}$.

For any matrix $\v{G}$ and its projection $\tilde{\v{G}}$ onto $\mathcal{V}(r,c)$, we have $\tilde{\v{G}}\v{X}\v{X}^T\tilde{\v{G}}^{*} = \v{D}^2$, where $\tilde{\v{G}}^{*}$ is the complex conjugate of $\tilde{\v{G}}$.
Consequently, given that the Frobenius norm is unitary invariant and the fact that $\tilde{\v{G}}$ is unitary: 

\begin{align*}
   \|\v{G}\v{X}\v{X}^T\v{G}^T - \v{D}^2\| 
   &= \|\tilde{\v{G}}\v{P}\v{X}\v{X}^T\v{P^T}\tilde{\v{G}}^* - \v{D}^2\|  \nonumber \\
   &= \|\tilde{\v{G}}\v{P}\v{X}\v{X}^T\v{P}^T\tilde{\v{G}}^* - \tilde{\v{G}}\v{X}\v{X}^T\tilde{\v{G}}^*\|  \nonumber \\
   &= \|\v{P}\v{X}\v{X}^T\v{P}^T - \v{X}\v{X}^T\|  \nonumber \\
   &\leq \|\v{P}^2 - \v{I}\| \ \|\v{X}\v{X}^T\|  \nonumber  
\end{align*}

The last line comes from the triangle inequality and that $\v{P}$ is Hermitian so $\v{P}\v{P}^T = \v{P}^2$.
As $\|\v{X}\v{X}^T\|$ is a constant, $\|\v{P}^2 - \v{I}\| \ \|\v{X}\v{X}^T\| \to 0$ as $\|\v{P}^2 - \v{I}\| \to 0$.

\begin{proposition}
Let $\v{G} \in \mathbb{R}^{r \times c}$ and $\tilde{\v{G}}$ is the projection of $\v{G}$ onto $\mathcal{V}(r,c)$. 
Further let  $\v{G} = \tilde{\v{G}}\v{P}$. Then $\|\v{G} - \tilde{\v{G}}\| \to 0 \Leftrightarrow \|\v{P}^2 - \v{I}\| \to 0$.
\label{prop:1}
\end{proposition}
\begin{proof}
As the Frobenius norm is invariant to unitary transformations, and because $\tilde{\v{G}}$ and $\v{U}$ are unitary matrices, the norm between $\tilde{\v{G}}$ and $\v{G}$ is
\begin{align}
  \|\v{G} - \tilde{\v{G}}\| 
  &= \|\v{G}\tilde{\v{G}}^* - \v{I}\|  \nonumber \\
  &= \|\v{U}\v{P}\v{U}^* - 
   \v{I}\|  \nonumber \\
  &= \|\v{U}(\v{P} - 
   \v{I})\v{U}^*\|  \nonumber \\
  &= \|\v{P} - 
   \v{I}\|.  \nonumber 
\end{align}

$\v{P}$ is positive-semidefinite and hence $\|\v{P}^2 -\v{I}\| \to 0$ implies $\|\v{P} -\v{I}\| \to 0$ and the distance between $\v{G}$ and $\mathcal{V}(|\v{s}|,|\v{z}|)$  strictly decreases, i.e $\|\v{G} - \tilde{\v{G}}\| \to 0$. 
Generally:
\begin{equation}
   \|\v{P}^2 -\v{I}\| \to 0
   \Leftrightarrow \|\v{G} - \tilde{\v{G}}\| \to 0
   \label{eq:limit}
\end{equation}
This ends the proof. 
\end{proof}

Thus by Proposition \ref{prop:1} we have that: 

\begin{equation}
\|\v{G} - \tilde{\v{G}}\| \to 0 \Leftrightarrow \|\v{G}\v{X}\v{X}^T\v{G}^T - \v{D}\| \to 0. 
\end{equation}

This ends the proof. 

\end{proof}

\section{Proof of closeness}
\label{app:close}

\begin{theorem}

Let $\v{G}\in \mathbb{R}^{d_s,D}$ and let $\tilde{\v{G}}$ be its projection onto $\mathcal{V}(d_s,d_x)$.  
As the Frobenius norm $\|\v{G}\v{G}^{T}-\v{I}\| \to 0$, we also have $\|\tilde{\v{G}}-{\v{G}}\| \to 0$. 

\end{theorem}

\begin{proof}

Let $\v{G}\in \mathbb{R}^{d_s,D}$ and let $\tilde{\v{G}}$ be its projection onto $\mathcal{V}(d_s,d_x)$. By Proposition \ref{prop:1} in Appendix \ref{app:optimal} we have that \[\|\v{P}^2 -\v{I}\| \to 0 \Leftrightarrow \|\v{G} - \tilde{\v{G}}\| \to 0\]

$\v{G} \in \mathbb{R}^{d_s \times d_x}$, has polar decomposition  $\v{U}\v{P}$.
Note that $\v{G}\v{G}^* = \v{U}\v{P}\v{P}^T\v{U}^* = \v{U}\v{P}^2\v{U}^*$ because $\v{P}$ is Hermitian. 
Recall that the Frobenius norm is invariant to unitary transformations, such as $\v{U}$: 
\begin{align}
   \|\v{G}\v{G}^* - \v{I}\| 
   &= \|\v{U}\v{P}^2\v{U}^* -\v{I}\|  \nonumber \\
   &= \|\v{P}^2 -\v{I}\|.  \nonumber 
\end{align}
As such: 
\begin{equation}
   \|\v{G}\v{G}^* - \v{I}\| \to 0
   \Leftrightarrow \|\v{G} - \tilde{\v{G}}\| \to 0
   \label{eq:limit_2}
\end{equation}

Note that we can trivially show this for $\v{G}^*\v{G} = \v{P}^2$.

This ends the proof. 
\end{proof}

\section{Linear ICA using JL Projections}
\label{app:linear_ica}

Instead of having $\v{W}=\v{\Sigma}^{-1}\v{U}^T$ be the result of SVD on the data matrix, we define our whitening matrix as $\v{W}=\v{\Lambda}\v{Q}$, a learnt diagonal matrix $\v{\Lambda}\in\mathbb{R}^{d_s}$ and a fixed, data-independent rectangular matrix $\v{Q}\in\mathbb{R}^{d_s\times d_x}$.
We must ensure that $\v{Q}$, our rectangular matrix, is approximately orthogonal, lying close to the manifold $\mathcal{V}(d_s,d_x)$.

As with the ICA model for Bijecta we construct approximately orthogonal matrices for $\v{Q}$ by way of Johnson-Lindenstrauss (JL) Projections \citep{Johnson1984, Dasgupta2003}.
A JL projection $\v{Q}$ for $\mathbb{R}^{d_x} \rightarrow \mathbb{R}^{d_s}$ is sampled at initialisation from a simple binary distribution with equal probability \citep{Achlioptas2003} such that: 
\begin{equation}
    Q_{i,j} = \pm 1/\sqrt{d_s} .
    \label{eq:jl_components}
\end{equation}
This distribution satisfies $\mathbb{E}[\v{Q}\v{Q}^{T}] = \v{I}$, and, as shown \citep{Achlioptas2003}, such a draw has $\v{Q}\v{Q}^T \approx \v{I}$. 
This gives us our factorised form for $\v{A}^+$,
\begin{equation}
\v{A}^{+}=\v{\Phi}\v{R\Lambda}\v{Q}.
\label{eq:decorr_new}
\end{equation}
\footnote{We use capitalised bold Greek letters for diagonal matrices, capitalised bold Roman letters otherwise}.

We can choose to fix $\v{Q}$ after initialisation such that optimisation occurs in the significantly smaller $\mathbb{R}^{\frac{1}{2}d_s\times (d_s + 3)}$ space, solely for the matrix $\v{\Phi R \Lambda}$, greatly simplifying our optimisation problem.
We draw $\v{Q}$ at initialisation and leave it fixed for the remainder of optimisation, making it a much more efficient method than SVD-based whitening.

By Theorems 1 and 2 we know that $\v{A}^{+}$ using $\v{Q}$ as a component will be close to the manifold of decorrelating matrices, Eq~\eqref{eq:decorr_new}. 
In the next section, we detail constraints on the only matrix that is optimised, $\v{\Phi R}$, which improve convergence.

\subsection{The $SO(d_s)$ Lie group for $\v{R}$}

Recall that $\v{R}$ is a square and orthogonal decorrelating matrix of dimensionality $\mathbb{R}^{d_s \times d_s}$.
As such, we can constrain $\v{R}$ to be in the orthogonal group $O(d_s)$, which has been shown to be optimal for decorrelating ICA sources. 
In the case of square ICA, with sufficient data, the maximum likelihood unmixing matrices are members of this group and will be reached by models confined to this group \citep{Everson1999}.

We want to perform unconstrained optimisation in learning our matrix, so we wish to use a differentiable transformation from a class of simpler matrices to $O(d_s)$.
One such transform is the Cayley transform \citep{Cayley1846}, which maps a given anti-symmetric matrix $\v{M}$ (i.e., satisfying $\v{M} = -\v{M}^T$) to the group of special orthogonal matrices $SO(d_s)$ with determinant 1. 
$SO(d_s)$ is the elements of the group $O(d_s)$ with determinant 1.
Unlike $O(d_s)$ it is path-connected, aiding optimisation.
As such, we propose defining our square unmixing matrix using the Cayley transform of the anti-symmetric matrix $\v{M}$,
\begin{equation}
\v{R} = (\v{I} - \v{M})^{-1}(\v{I} + \v{M}).
\end{equation}
This can be formulated as an unconstrained problem, easing optimisation, by defining $\v{M} =(\v{L}-\v{L}^T)/2$ and then optimising over the square real-valued matrix $\v{L}$.
This further reduces the optimisation space for $\v{A}^+$ to $\mathbb{R}^{\frac{1}{2}d_s(d_s+3)}$.

\section{Sub- or Super-Gaussian Sources?}

In classical non-linear ICA a simple non-linear function (such as a matrix mapping followed by an activation function) is used to map directly from data to the setting of the sources for that datapoint \citep{Bell1995}.
In this noiseless model, the activation function is related to the prior one is implicitly placing over the sources \citep{Roweis1999}.
The choice of non-linearity here is thus a claim on whether the sources were sub- or super- Gaussian.
If the choice is wrong, ICA models struggle to unmix the data \citep{Roweis1999}.
Previous linear ICA methods enabled both mixed learning of super and sub sources \citep{Lee1997}, and learning the nature of the source \citep{Everson1999}.

\section{Identifiability of Linear ICA}
Recall that for noiseless linear ICA the learnt sources will vary between different trained models only in their ordering and scaling \cite{Choudrey2000, Hyvarinen2001, Everson2001}.
Under this model our data matrix $\v{X}$ is
\begin{equation}
    \v{X} = \v{A}\v{S},
\end{equation}
where each column of $\v{S}$ is distributed according to $p(\v{s})$.
Given a permutation matrix $\v{P}$ and a diagonal matrix $\v{D}$, both $\in\mathbb{R}^{d_s\times d_s}$, we define new source and mixing matrices such that $\v{X}$ is unchanged:
$\v{A} \leftarrow \v{A}\v{P}\v{D}$ and $\v{S} \leftarrow \v{P}\v{D}^{-1}\v{S}$ \cite{Choudrey2000}.
With fixed source variance the scaling ambiguity is removed, so linear ICA is easy to render identifiable up to a permutation and sign-flip of sources.
However, this is only the case when the underlying sources are non-Gaussian \citep{Hyvarinen2001}.
The Bijecta model can be interpreted as learning the ``data'' for linear ICA, with
\begin{align}
f(\v{X}) \approx \v{A}\v{S}.
\end{align}
In this setting, a natural question to ask is whether or not, given a particular function $f : \mathcal{X} \rightarrow \mathcal{Z}$,
the linear ICA is identifiable.
The Bijecta model explicitly aims to induce this linear identifiability on its feature map, as we impose a Laplace prior $p(\v{s})$, with fatter-than-Gaussian tails.

\section{Coupling Layers in Flows}
\label{app:coupling-layers}

Coupling layers in flows are designed to produce lower triangular Jacobians for ease of calculation of determinants.
RQS-flows are defined by $K$ monotonically increasing \textit{knots}, which are
the coordinate pairs through which the function passes: $\{(x_k,y_k)\}^K_{k=0}$. 
We can interpolate values between each of the $K$ knots using the equation for the RQS transformation~\cite{Durkan2019}.
The resulting function is a highly flexible non-linear transformation, such that RQS flows require fewer composed mappings to achieve good performance relative to other coupling layers. %
The knots themselves are trainable, and parameterised by deep neural networks.
These can then be composed with other tractable transformations, including permutations and multiplication by triangular matrices.
More specifically these layers can be defined as \citep{Durkan2019}: 
\begin{enumerate}
    \item Given an input $\v{x}$, split $\v{x}$ into two parts $\v{x}_{1: d-1}$ and $\v{x}_{d: D}$;
    \item Using a neural network, compute parameters for a bijective function $f$ using one half of $\v{x}$: $\theta_{d:D} = \mathrm{NN}(\v{x}_{1:d-1})$;
    parameters $\theta_{1:d-1}$ are learnable parameters that do \textit{not} depend on the input;
    \item The output $\v{y}$ of the layer is then $y_i=f_{\theta_i}(x_i)$ for $i=1,...,d_x$. 
\end{enumerate}
These coupling transforms thus act elementwise on their inputs.

\section{Network Architectures and Hyperparameters}
\label{app:arch}
\subsection{RQS flows and Bijecta}
Within all Rational Quadratic Spline (RQS) flows we parameterise 4 knots for each spline transform.
The hyper-parameters of the knots were as in the reference implementation from \cite{Durkan2019}, available at \href{https://github.com/bayesiains/nsf}{github.com/bayesiains/nsf}:
we set the minimum parameterised gradient of each knot to be 0.001, the minimum bin width between each encoded knot and the origin to be 0.001, and the minimum height between each knot and the origin to be 0.001.

Unlike in \cite{Durkan2019}, where they view a single RQS `layer' as composed of a sequence of numerous coupling layers, in this paper the number of layers we describe a model as having is exactly the number of coupling layers present.
So for our 4-layer models there are four rational quadratic splines.
Each layer in our flows are composed of: an actnorm layer, an invertible 1x1 convolution, an RQS coupling transform and a final 1x1 invertible convolution.

The parameters of the knots were themselves parameterised using ResNets nets, as used in RealNVP \cite{Dinh2017}, for each of which we used 3 residual blocks and batch normalisation.
As in \cite{Dinh2017} we factor-out after each layer.
All training was done using ADAM \cite{Kingma2015}, with default $\beta_1, \beta_2$, a learning rate of 0.0005 and a batch size of 512.
We perform cosine decay on the learning rate during training, training for 25,000 steps.

Data was rescaled to 5-bit integers and we used RealNVP affine pre-processing so our input data was in the range $[\epsilon,1-\epsilon]$ with $\epsilon=0.05$.

\subsection{VAEs}

\paragraph{Affine}
For the affine experiments we used fully convolutional encoders and decoders, each made out of 5 resnet block.
Each block also down/up-scales their input by a factor of 2 along the spatial input dimensions.
As our input images are $32\times 32$, this means that 5 such scalings map to $1\times 1$ representations.

As we pass through the encoder we double the number of features for each block, from an initial number of 16.
Thus the final residual output has $32\times 16=512$ filters.
A $1\times1$ convolutional layer then maps this to the posterior.
The decoder performs these same operations in reverse order: an initial $1\times 1$ convolutional layer maps the sampled value of the latent variable into a $1\times 1\times 512$ hidden representation, a chain of 5 upscaling resnet blocks map this to a $32\times 32\times 16$ representation.
Finally a $1\times 1$ convolution maps this to the $32\times 32\times 3$ sub-pixel means of the likelihood function.

\paragraph{CelebA}
For CelebA we used the same networks as used by \cite{Chen2018} for their CelebA experiments.
The encoder and decoder are composed for 5 convolutions/transposed convolutions, with batchnorm layers in between.
The number of filters increases as we go up the encoder: $64 \rightarrow 64 \rightarrow 128 \rightarrow 128 \rightarrow 512$ before being mapped to the posterior distribution's parameters by a $1 \times 1$ convolution.
The decoder has the same sequence of filter sizes in its transposed convolutions but in reverse, with a final convolution to the $32\times 32\times 3$ sub-pixel means of the likelihood function.

\newpage
\onecolumn
\section{Reconstructions, Latent Traversals, and Samples}
\subsection{Numerical Results}
\label{app:num_results}

\begin{table*}[h]
  \caption{Here we evaluate the source-separation and reconstruction quality of non-square ICA models. 
  We evaluate source separation by evaluating the mean log probability of the validation set embeddings in $\mathcal{S}$ under our heavy-tailed prior, normalised by the dimensionality of $\mathcal{S}$ space: $\log p(\v{s})/d_s$ ($d_s=64$).
  As our base distribution is heavy-tailed, this metric evaluates the axis-alignment, the \textit{independence}, of learnt factors. 
  We also give the validation set bits-per-dim (bpd), a scaled ELBO and thus a measure of model quality (lower better).
  Bijecta models consistently have lower bpd than the linear models.
  We evaluate the quality of low-dimensional representations by measuring the $L_1$ reconstruction error in $\mathcal{X}$.
  A better representation encodes more information in $\mathcal{S}$, making it easier for the model to then reconstruct in $\mathcal{Z}$ and subsequently in $\mathcal{X}$.
  Most striking is the improvement across all metrics when introducing a \textit{single} bijective mapping. 
  Our 4-layer model further improves the quality of the compressed representations as seen by the lower reconstruction errors.}
  
  \label{non-square-res}
  \centering
  \resizebox{\textwidth}{!}{%
  \begin{tabular}{r|rcccc}
    \toprule
                            && CIFAR-10 & MNIST & fashion-MNIST & CelebA\\
        \midrule
    \multirow{3}{*}{Linear-ICA}    
                                & $\log p(\v{s})/d_s$ & -4.8   & -4.29 & -3.0 & -9.0\\
                                & bits-per-dim   & 6.3   & 8.6   & 6.6 & 6.2\\
                                & $L_1$ reconstruction error in $\mathcal{X}$   & 3.0   & 3.1   & 2.9 & 2.9\\
                                \midrule
    \multirow{3}{*}{1-layer Bijecta}    
                                &  $\log p(\v{s})/d_s$ &  \textbf{-4.2} & \textbf{-2.8} & \textbf{-2.2} & \textbf{-7.3} \\
                                & bits-per-dim  & 3.8  & 3.2 & 3.8 & 3.4\\
                                & $L_1$ reconstruction error in $\mathcal{X}$   & 2.0   & 1.0   & 1.6 & 1.9\\
            \midrule
     \multirow{3}{*}{4-layer Bijecta}    
                                &  $\log p(\v{s})/d_s$ &  -4.7 & -3.0 & \textbf{-2.2} & -7.8\\
                                & bits-per-dim  & \textbf{3.2}  & \textbf{2.1} & \textbf{3.1} & \textbf{3.2}\\
                                & $L_1$ reconstruction error in $\mathcal{X}$   & \textbf{1.9}   & \textbf{0.6}   & \textbf{1.2} & \textbf{1.4}\\
    \bottomrule
  \end{tabular}}
\end{table*}

\newpage
\subsection{Reconstructions}
\begin{figure}[h!]
\centering
\makebox[\textwidth]{
\subfloat[F-MNIST Input Data\label{subfig:fmnist_inputs}]{%
       \includegraphics[width = 4cm]{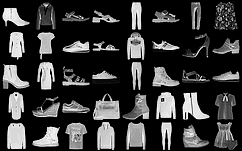}
     }
     \hspace{0mm}
\subfloat[F-MNIST Linear ICA Recon \label{subfig:fmnist_ica}]{%
       \includegraphics[width = 4cm]{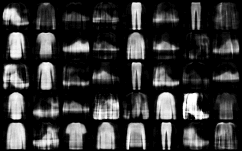}
     }
\subfloat[F-MNIST 1-Bijecta Recon \label{subfig:fmnist_bijecta1}]{%
       \includegraphics[width = 4cm]{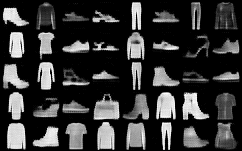}
     }
\subfloat[F-MNIST 4-Bijecta Recon \label{subfig:fmnist_bijecta3}]{%
       \includegraphics[width = 4cm]{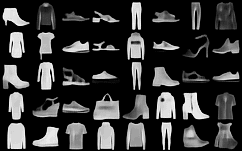}
     }
    }

\makebox[\textwidth]{
\subfloat[CIFAR-10 Input Data\label{subfig:cifar10_inputs}]{%
       \includegraphics[width = 4cm]{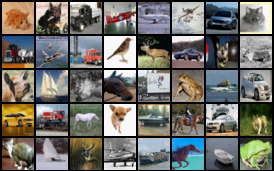}
     }
     \hspace{0mm}
\subfloat[CIFAR-10 Linear ICA Recon \label{subfig:cifar10_ica}]{%
       \includegraphics[width = 4cm]{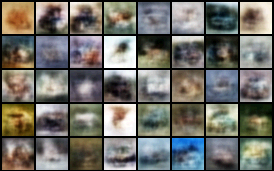}
     }
\subfloat[CIFAR-10 1-Bijecta Recon \label{subfig:cifar10_bijecta1}]{%
       \includegraphics[width = 4cm]{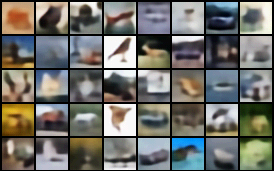}
     }
\subfloat[CIFAR-10 4-Bijecta Recon \label{subfig:cifar10_bijecta3}]{%
       \includegraphics[width = 4cm]{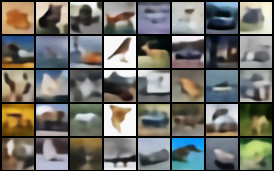}
     }
    }
\makebox[\textwidth]{
\subfloat[CelebA Input Data\label{subfig:celeba_inputs}]{%
       \includegraphics[width = 4cm]{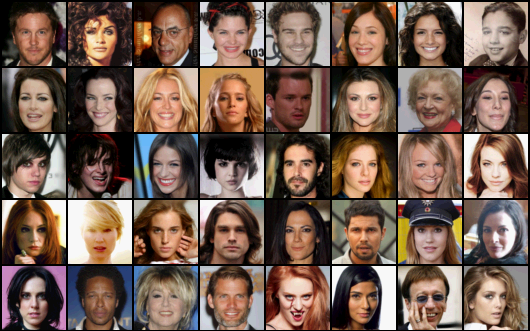}
     }
     \hspace{0mm}
\subfloat[CelebA Linear ICA Recon \label{subfig:celeba_ica}]{%
       \includegraphics[width = 4cm]{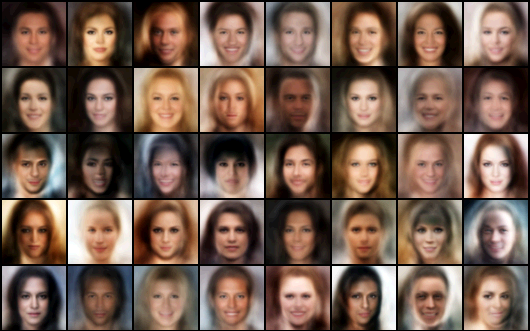}
     }
\subfloat[CelebA 1-Bijecta Recon \label{subfig:celeba_bijecta1}]{%
       \includegraphics[width = 4cm]{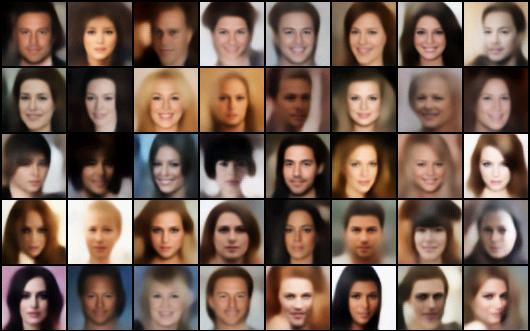}
     }
\subfloat[CelebA 4-Bijecta Recon \label{subfig:celeba_bijecta3}]{%
       \includegraphics[width = 4cm]{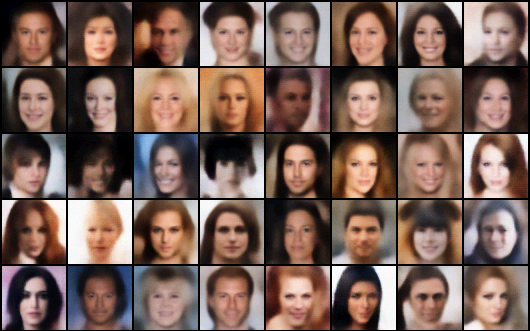}
     }
    }
     
\caption{Original data and reconstructions for Fashion MNIST (a)-(d), CIFAR-10 (e)-(h) and CelebA (i)-(l). The first column -- (a), (e), (i) -- shows a random selection of 40 images from each dataset's training set. The second column -- (b), (f), (j) -- shows the reconstructions obtained by linear non-square ICA using our approximately-Stiefel unmixing matrix as is Appendix \ref{app:linear_ica}, with $d_s=64$. The third column -- (c), (g), (k) -- shows the reconstructions for Bijecta with an RQS flow with a single layer and $d_s=64$. The fourth column -- (d), (h), (l) -- for a 4 -layer Bijecta, $d_s=64$. Bijecta models show much higher fidelity reconstructions than the linear ICA model. On CelebA the 4-layer model gives the highest fidelity reconstructions. }
\label{fig:app_recons}
\end{figure}

\begin{figure}[h!]{
    \centering
\subfloat[CelebA Input Data]{
       \includegraphics[width=4cm]{figures/appendix/cropped/celeba_originals.png_cropped.png}
     }
\subfloat[CelebA 12-Bijecta Recon]{%
       \includegraphics[width=4cm]{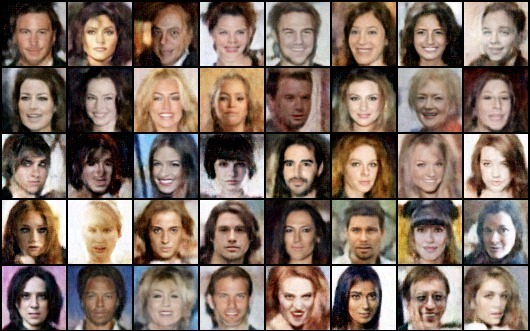}
     }

\caption{Here we show reconstructions for a 12-layer Bijecta model, trained on a batch size of 64 for the CelebA dataset, with a latent space dimensionality $d_s=512$. The quality of the reconstructions clearly illustrate that as we stack invertible layers our model and increase the size of $\mathcal{S}$ we are able to reconstruct images with a high degree of accuracy.}}
\label{fig:large_recons_12layer}
\end{figure}

\newpage
\subsection{Latent Traversals}
\label{app:traversals}

\begin{figure}[h!]
\centering
\makebox[\textwidth]{

\subfloat[Dim 3 - Gender]{%
       \includegraphics[width=6cm]{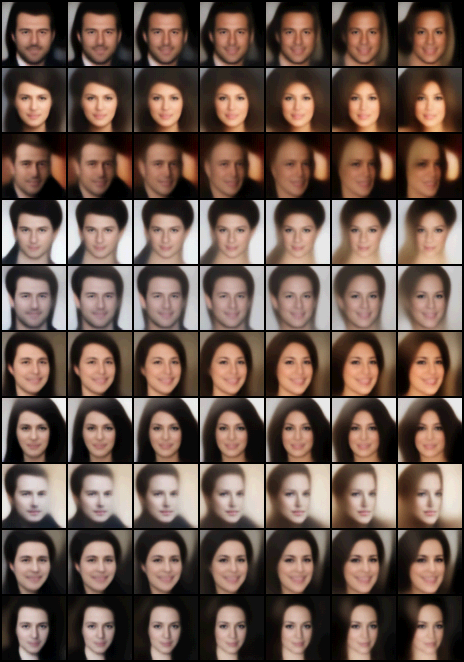}}

\subfloat[Dim 12 - Face Rotation]{%
       \includegraphics[width=6cm]{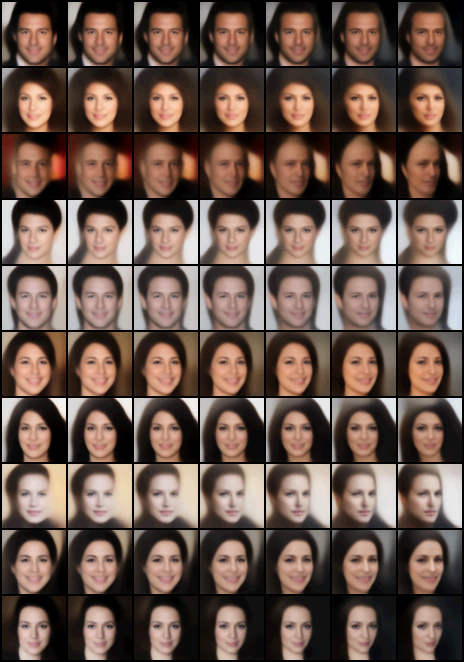}}
     }\\
\makebox[\textwidth]{

\subfloat[Dim 15 - Light Warmth]{%
       \includegraphics[width=6cm]{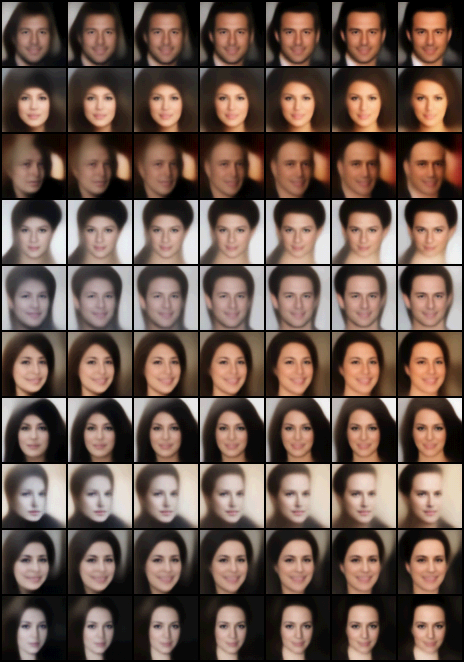}}

\subfloat[Dim 24 - Eye Shadow]{%
      \includegraphics[width=6cm]{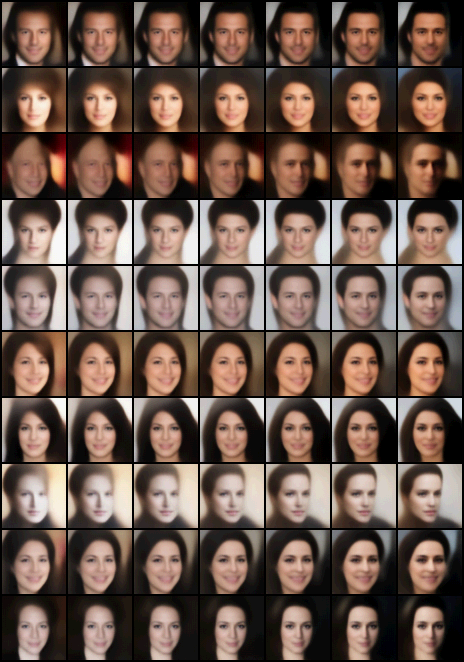}}     
       
\subfloat[Dim 27 - Face Width]{%
       \includegraphics[width=6cm]{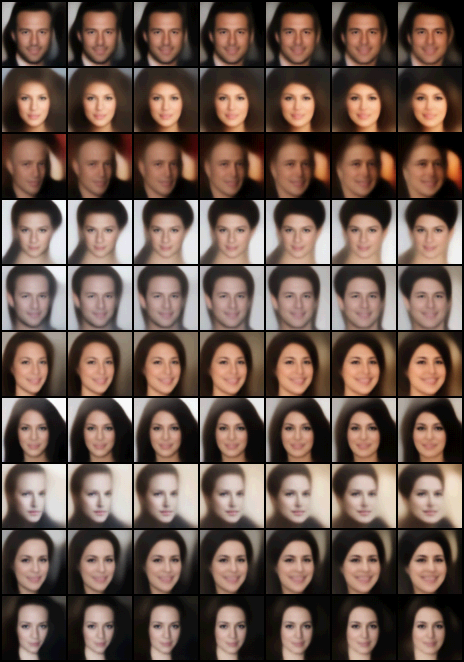}}
     }

\caption{Here we show latent traversals along 5 of 32 of the latent dimensions for 10 posterior samples from an 8-layer Bijecta model. 
We list the dimension's encoded factor of variation under each image. }
\label{fig:app_latent_traversals}
\end{figure}

\newpage
\subsection{Samples}

\subsubsection{Marginal Posterior Samples CelebA}
\label{app:kde_samples}

\begin{figure*}[h]
    \centering
     \includegraphics[width=0.35\textwidth]{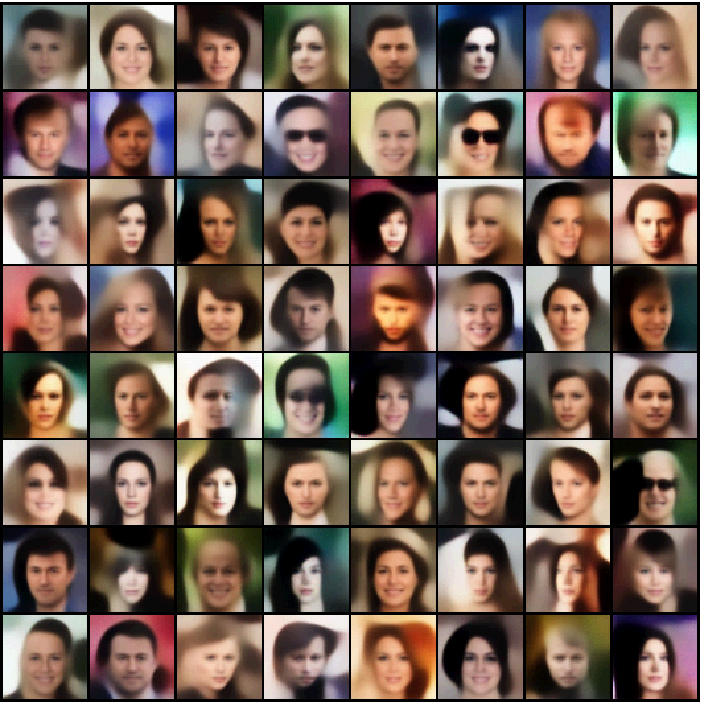}
      \hspace{0.5cm}
    \caption{Here we show decodings from an 8-layer Bijecta ($d_s=32$) trained on CelebA where we sample from a factorised approximation to Bijecta's posterior.}
    \label{fig:post-samples-celeba}
\end{figure*}

\subsubsection{Prior Samples CelebA}
\label{app:prior_samples}
\begin{figure}[h!]
\centering
       \includegraphics[width=18cm]{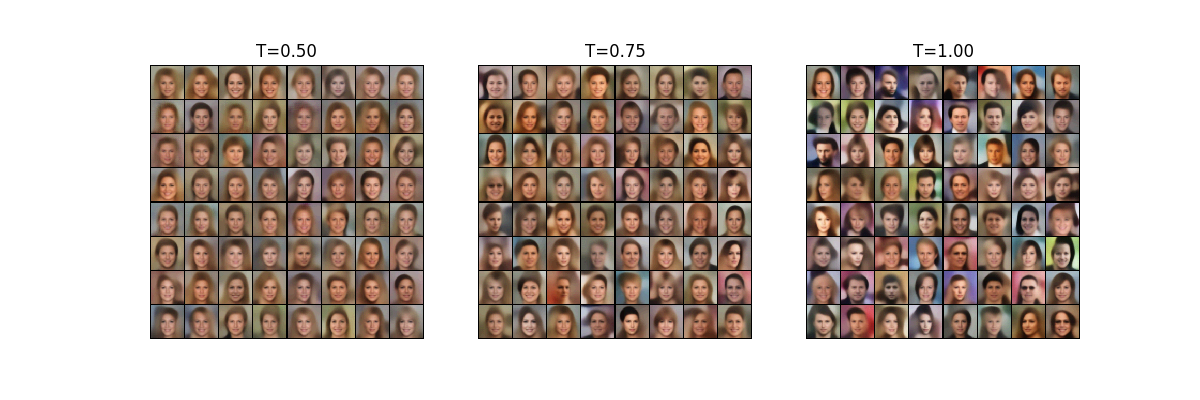}
       \\
       \includegraphics[width=18cm]{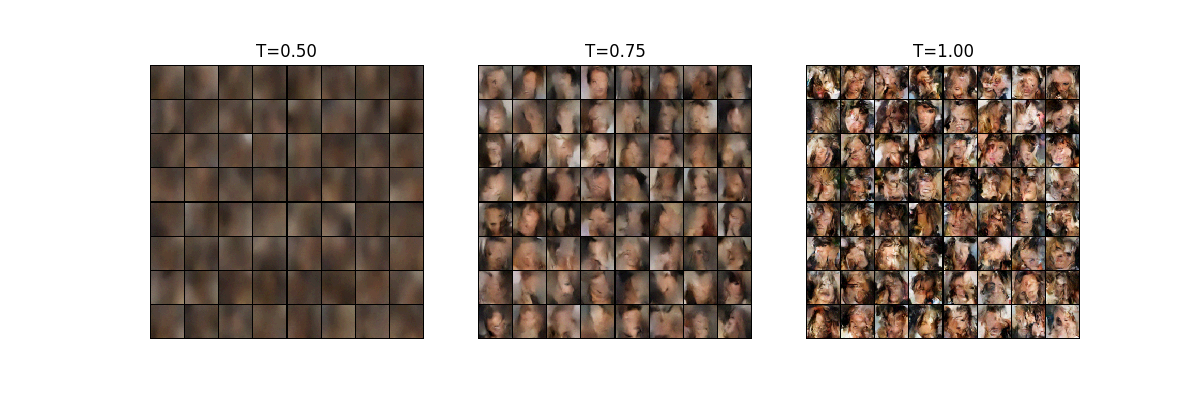}
\caption{Here we show samples from an 8-layer Bijecta model (top) and an 8-layer RQS-flow (bottom) with Laplace prior, trained on a batch size of 256 for the CelebA and CIFAR-10 datasets for 100000 steps, with a latent space dimensionality $d_s=32$. We show samples rescaled at scale T. }
\label{fig:samples_bijecta}
\end{figure}

\end{document}